\documentclass{article}

\usepackage{aistats_modified}

\usepackage{natbib}
\usepackage{graphicx}
\usepackage[ruled,vlined,algosection]{algorithm2e}
\usepackage{hyperref}
\usepackage{hyperref}       
\usepackage{url}            
\usepackage{booktabs}       
\usepackage{nicefrac}       
\usepackage{microtype}      
\usepackage{amsthm}
\usepackage{amsmath}
\usepackage{amssymb}
\usepackage{amsfonts}       
\usepackage[dvipsnames]{xcolor}
\usepackage[nameinlink,capitalize,noabbrev]{cleveref}
\usepackage{tcolorbox}
\usepackage{graphicx}
\usepackage{array}
\usepackage{bm}
\usepackage{eucal}          

\numberwithin{equation}{section}

\renewcommand{\cite}[1]{\citep{#1}}   
\newcommand{\citeasnoun}[1]{\citet{#1}}

\hypersetup{colorlinks=true,citecolor=Gray,linkcolor=Black,urlcolor=Black}
\newcommand{\norm}[1]{\left\lVert#1\right\rVert}
\newcommand{\real}{\mathbb{R}}
\newcommand{\states}{\mathcal{S}}
\newcommand{\actions}{\mathcal{A}}
\newcommand{\opt}{^\star}

\newcommand{\tr}{^\top}
\newcommand{\data}{\mathcal{D}}
\newcommand{\cs}{\\[1ex] & }

\newcommand{\maximize}[1]{\operatorname*{maximize}_{#1} &}

\newcommand{\one}{\bm{1}}
\newcommand{\zero}{\bm{0}}

\newcommand{\PP}{\mathcal{P}}
\renewcommand{\ss}{\,\mid\,}
\newcommand{\srob}{\rho^S}
\newcommand{\drob}{\rho^D}
\newcommand{\rrob}{\rho^R}
\newcommand{\sarob}{\rho^{RA}}

\newcommand{\sr}{R}				

\newcommand{\Real}{\mathbb{R}}

\newcommand{\E}[1]{\mathbb{E}\left[ #1 \right]}
\newcommand{\Ex}[2]{\mathbb{E}_{#1}\left[ #2 \right]}

\renewcommand{\Pr}[2]{\mathbb{P}_{#1}\left[ #2 \right]}

\newcommand{\eye}{\bm{I}}

\newcommand{\Pf}{f}
\newcommand{\st}{\operatorname{s.t.}}
\newcommand{\stc}{\operatorname{subject\,to} &}
\newcommand{\Exp}[2]{\mathbb{E}_{#1} \left[ #2 \right] }
\renewcommand{\Pr}{\mathbb{P}}

\newcommand{\Bell}{\mathfrak{T}}

\newcommand{\Topt}{\mathfrak{T}}

\newcommand{\w}{w}

\theoremstyle{plain}
\newtheorem{thm}{Theorem}[section]
\newtheorem{cor}[thm]{Corollary}
\newtheorem{lem}[thm]{Lemma}
\newtheorem{prop}[thm]{Proposition}

\theoremstyle{definition}

\theoremstyle{remark}

\DeclareMathOperator{\diag}{diag}

\renewcommand{\Pr}{\mathbb{P}}

\usepackage{xifthen}
\usepackage[normalem]{ulem}
\newcommand{\mm}[2][]{\ifthenelse{\isempty{#1}}{}{\textcolor{red}{[\sout{#1}]}}\textcolor{green}{#2}}

\crefmultiformat{thm}{Theorems~#2#1#3}{ and~#2#1#3}{, #2#1#3}{ and~#2#1#3}

\newenvironment{mprog}{\begin{array}{>{\displaystyle}r>{\displaystyle}l>{\displaystyle}l}}{\end{array}}

\DeclareMathOperator{\cvar}{CVaR}
\DeclareMathOperator{\var}{VaR}

\begin{document}

\twocolumn[
\aistatstitle{Soft-Robust Algorithms for Batch Reinforcement Learning}
\aistatsauthor{ Elita A. Lobo \And Mohammad Ghavamzadeh \And Marek Petrik }
\aistatsaddress{ University of Massachusetts, Amherst\\\url{elitalobo@cs.umass.edu} \And  Google Research\\\url{ghavamza@google.com} \And University of New Hampshire \\\url{mpetrik@cs.unh.edu}}
]

\begin{abstract}
In reinforcement learning, robust policies for high-stakes decision-making problems with limited data are usually computed by optimizing the percentile criterion, which minimizes the probability of a catastrophic failure. Unfortunately, such policies are typically overly conservative as the percentile criterion is non-convex, difficult to optimize, and ignores the mean performance. To overcome these shortcomings, we study the soft-robust criterion, which uses risk measures to balance the mean and percentile criterion better. In this paper, we establish the soft-robust criterion's fundamental properties, show that it is NP-hard to optimize, and propose and analyze two algorithms to approximately optimize it. Our theoretical analyses and empirical evaluations demonstrate that our algorithms compute much less conservative solutions than the existing approximate methods for optimizing the percentile-criterion.
\end{abstract}

\section{Introduction}
\label{sec:intro}

Markov Decision Process (MDP) is an established model for optimizing returns in sequential decision-making problems~\citep{putterman1994DP,Sutton2018RL}. In the batch Reinforcement Learning (RL) setting, MDPs must be estimated from logged data. However, without the ability to explore, the transition probability estimates derived from the logged data are inevitably imprecise. Such errors in the transition model estimate often result in learning policies that fail catastrophically when deployed. In this paper, we aim to compute robust policies from logged data in a way that accounts for the uncertainty in transition models.
As is common in prior work, we use parametric Bayesian models to represent the uncertainty in the transition model estimates~\citep{Xu2006Robustnesstradeoff,Xu2009Parametric,Xu2012Distributionally,Delage2009Percentile,Petrik2016SafePI,russel2019BeyondCR}. The most common robust objective in this setting is the \emph{percentile criterion}, which maximizes a given $\alpha$-quantile of the expected returns~\citep{Delage2009Percentile,Tamar2014Scaling,Chow18RC,russel2019BeyondCR}.

Despite its simplicity and popularity, the percentile criterion objective suffers from three major shortcomings. {\em First}, it ignores the mean performance even when there are multiple optimal random policies~\citep{Iancu2014Pareto}. This behavior gives rise to policies that are unnecessarily pessimistic.
{\em Second}, the percentile criterion also ignores the tail of the distribution below the $(1-\alpha)$ quantile in addition to ignoring the mean. In problems with heavy tail risks, such as some portfolio optimization settings~\cite{Krokhmal2003PO}, the percentile criterion learns over-optimistic policies that may result in disastrous worst-case outcomes~\citep{Rockafellar2000optimizationCVAR}. {\em Third}, the percentile criterion is non-convex which complicates its analysis and optimization. Optimizing this criterion using robust optimization methods~\cite{Ben-Tal2019SoftRobust} requires constructing a convex uncertainty set that accounts for $\alpha\%$ of the model parameter values~\citep{Wiesemann2013RMDP,russel2019BeyondCR}. In practice, these sets are constructed using statistical confidence intervals. Recent work~\cite{russel2019BeyondCR,gupta2019NearOpt} has shown that such uncertainty sets (that are confidence regions) are often very large and lead to overly-conservative policies.

To overcome the limitations of the percentile criterion, we adopt the \emph{soft-robust criterion}~\citep{Ben-Tal2019SoftRobust}. This criterion optimizes a convex combination of the mean and a robust performance and is itself convex. We measure the robust performance in soft-robust criterion using the Conditional Value at Risk ($\cvar$) measure~\citep{Rockafellar2000optimizationCVAR}, which represents the mean of the expected return of the worst $(1-\alpha)\%$ of the models. The $\cvar$ measure bounds the percentile criterion from below and takes into account the tail risk. We note that although the soft-robust criterion has been widely studied in finance and risk-averse RL, (see e.g.,~\citealt{Prashant2014PG,Chow2014CVAR,tamar15PG,tang2019worst}), it is not well-understood in the context of robust RL.
We discuss related work in greater depth in \cref{sec:Related work}.

We begin by analyzing a new \emph{static} soft-robust formulation for RL, which differs significantly from earlier formulations~\citep{Xu2012Distributionally,Derman2018Soft-RobustAC,Mankowitz2020RobustRL}. Since the earlier formulations embed the robustness within the dynamic programming equations, we refer to them as robust objectives with \emph{dynamic uncertainty model}. Dynamic uncertainty model can be interpreted, in certain cases, as an assumption that the uncertain transition model changes randomly in every time step~\cite{Xu2012Distributionally}. Our {\em static uncertainty model}, in contrast, assumes that the model is uncertain but it does not change throughout the execution. We show that, despite being computationally challenging, the static uncertainty model has two important advantages. {\em First}, the static uncertainty model is less conservative than the dynamic one because it allows the agent to exploit any information about the uncertain parameters to make better decisions. {\em Second}, because the static uncertainty model accounts for model uncertainty more accurately, it effectively eliminates over-optimism driven by model uncertainty, also known as the \emph{optimizer's curse}~\cite{smith2006Optimizer}.


In addition to describing the static soft-robust criterion, we derive two new algorithms for optimizing it in \cref{sec:Soft-Robust-optimization}. The first algorithm is a new mixed-integer linear program (MILP) formulation that computes the optimal deterministic soft-robust policy. While this non-convex formulation obviously does not scale beyond small problems, it is unlikely that more tractable optimal algorithms exist, because the soft-robust objective is NP hard. The second algorithm approximates the static objective by a robust MDP and scales to continuous problems using value function approximation. Finally, we derive a new structural approximation error bounds for the robust MDP formulation in \cref{sec:approximation_error}. Our experimental results in \cref{sec:experiments} illustrate the algorithms on two small, but realistic, problem domains.


\section{Preliminaries}
\label{sec:motivation}

We model the agent's interaction with the environment as an MDP~\citep{putterman1994DP}. An MDP is a tuple $( \states,\actions,P,r,p_0,\gamma )$ that consists of a set of states $\states = \{1,2,\ldots, S \}$, a set of actions $\actions = \{1,2,\ldots, A\}$, a reward function $r : \states \times \actions \times \states \to \real$, a transition probability function $P : \states\times\actions \to \Delta^S$, an initial state distribution {$p_0\in\Delta^S$}, and a discount rate $\gamma \in (0,1)$. The symbol $\Delta^S$ denotes the $S$-dimensional probability simplex. Our objective is to maximize the infinite-horizon discounted return. We also assume that $\lvert r(s,a,s') \rvert \le r_{\max} \in \Real$ for all $s,s'\in\states$ and $a\in\actions$.

We consider a \emph{batch} RL setting (e.g.,~\citealt{Lange2012BatchRL}), where the reward function $r(s,a,s')$ is known but the true transition model $P\opt$ is unknown and must be estimated from the batch of data $\mathcal{D} = \{(s_i,a_i,s_{i}')\}_{i = 1}^M$, where $s_i'\sim P\opt(s_i,a_i,\cdot)=p\opt_{s_i,a_i}$. We take a parametric Bayesian approach to model the uncertainty over the true transition model $P\opt$~\citep{Delage2009Percentile,Xu2012Distributionally,russel2019BeyondCR,Derman2019BayesianRobust}. In this approach, the transition model {$P\opt$} is a random variable. Using the batch of data $\mathcal{D}$, one can derive a posterior distribution over $P\opt$ conditional on $\mathcal{D}$, which is denoted by $\hat{P} = P\opt \ss \mathcal{D}$ and distributed according to a measure $\Pf$. As it is common in methods like sample average approximation (SAA)~\cite{Shapiro2014StochasticProgramming}, we approximate $\hat{P}$ by finite samples {$\hat{P}^{\omega} ,\, \omega \in \Omega$} with weights $\Pf_\omega,\, \omega \in \Omega$ and sample size $N = \lvert\Omega\rvert$. The samples $\hat{P}^\omega$ in our experiments come from MCMC and can be computed using tools like Stan or PyMC3~(e.g.,~\citealt{Gelman2014BayesianData}).
\nocite{Kallus2020}

A policy $\pi : \states \to \Delta^A$ prescribes the probability of taking an action $a\in\mathcal A$ when the agent is in a state $s \in \states$. We denote by $\Pi = (\Delta^{A})^{\states}$ and $\Pi_{D} = \actions^\states$, the sets of all randomized and deterministic policies, respectively. For a given realization of transition model $P$, the initial state distribution $p_0$, and a policy $\pi$, the expected discounted return is defined as
\begin{equation*}
\rho(\pi,P) \;=\; \mathbb E\Biggl[\sum_{t=0}^\infty \gamma^t \cdot r(S_t,A_t,S_{t+1})\Biggr]~,
\end{equation*}
where $S_0 \sim p_0$, $A_t \sim \pi(S_t)$, and $S_{t+1} \sim P(S_t,A_t,\cdot)$.

\paragraph{Percentile Criterion and Risk Measures.}
Percentile optimization has been commonly used to derive robust policies and risk-adjusted discounted returns for an MDP under uncertainty~\citep{Delage2009Percentile,russel2019BeyondCR}. The chance-constrained objective that this criterion optimizes is of the form
\begin{equation} \label{eq:percentile}
\max_{\pi \in \Pi,\, y \in \real} \; \Bigl\{y \ss \Pr_{\hat{P}\sim\Pf}[ \rho(\pi,\hat{P}) \geq y ] \ge \alpha \Bigr\}\,,
\end{equation}
where $y$ lower-bounds the true expected discounted return with confidence $\alpha\in [0,1]$. Increasing the value of $\alpha$ in~\eqref{eq:percentile} increases the confidence that the return $\rho(\pi,\hat{P})$ is at least $y$. Alternatively, the percentile criterion in~\eqref{eq:percentile} can be interpreted using the framework of risk measures as:
%
\begin{equation} \label{eq:percentile_var}
\max_{\pi \in \Pi} \, \var_{\hat{P}}^{\alpha} \left[\rho(\pi,\hat{P})\right]\,,
\end{equation}
where $\var$ is the well-known value at risk measure~\cite{Shapiro2014StochasticProgramming} that is defined for a random variable $Z$ with PDF $g$ and CDF $G$ as $\var^{\alpha}_g[Z] = \inf\{z \in \real \ss G(z) \geq 1-\alpha\}$.

\paragraph{Robust MDPs.}
Robust MDPs~(RMDPs)~(e.g.,~\citealt{iyengar2005RMDP,Wiesemann2013RMDP}) are commonly used to optimize the percentile criterion~\cite{Delage2009Percentile,Xu2012Distributionally,Petrik2016SafePI,russel2019BeyondCR}. We will also use them in this paper to approximately optimize the soft-robust criterion. RMDPs generalize MDPs by allowing for ambiguous transition models. An RMDP consists of the same components as an MDP, except the fixed transition function $P$ is replaced by an ambiguity set $\mathcal{P} \subseteq \{ P : \states \times \actions \to \Delta^S  \}$ of plausible transition models. The objective is to compute a policy $\pi\in\Pi$ that maximizes the return for the worst-case realization of $P \in \PP$, i.e.,
\begin{equation} \label{eq:rmdp_general}
\max_{\pi \in \Pi}\min_{P \in \mathcal{P}} \; \rho(\pi,P)~.
\end{equation}
Although solving \eqref{eq:rmdp_general} is NP-hard~\cite{Bagnell2004LearningDR}, it is tractable for certain classes of ambiguity set $\PP$, including SA-, S-, R-, and K-rectangular~\cite{iyengar2005RMDP,letallec2007Robust,Mannor2016KRect,goyal2020robust}.
We focus on S-rectangular sets in this paper because they are both general and tractable~\citep{Wiesemann2013RMDP}. All our results can be extended to SA-rectangular sets~\cite{Wiesemann2013RMDP} as shown in~\cref{app:sa_rectangular}. An ambiguity set $\PP$ is S-rectangular if
\begin{equation} \label{eq:s-rect}
    \PP = \left\{p \in (\Delta^S)^{S \times A} \ss (p_{s,a})_{a\in\actions} \in \mathcal{P}_s, \; \forall s \in \states \right\},
\end{equation}
for some $\mathcal{P}_s \subseteq (\Delta^S)^{\actions}$. Each element $p\in\mathcal{P}_s$ is a function $p:\actions\rightarrow \Delta^S$ that determines the transition probabilities for all actions $a\in\actions$ at the state $s$. The intuitive interpretation is that the adversary can choose the worst-case transition probability for each state independently.

In S-rectangular RMDPs, the optimal value function $v\opt\in\real^S$ exists, is unique, and is the fixed-point of the S-rectangular robust Bellman optimality operator $\Bell_{\mathcal{P}} : \real^S \to \real^S$ that is defined for each $v\in\real^S$ and $s\in\states$ as~\cite{iyengar2005RMDP,Wiesemann2013RMDP}
\begin{equation} \label{eq:S-rect}
(\Bell_{\mathcal{P}} \, v)(s) = \max_{d\in\Delta^A} \min_{p \in \mathcal{P}_s } \; \sum_{a\in\actions} d_a \cdot \left(r_{s,a}  + \gamma \cdot p_{a}\tr v \right).
\end{equation}
The randomized decision rule $d$ in~\eqref{eq:S-rect} can be used to construct the optimal randomized policy. The optimal value function $v\opt$ can be computed using either value iteration~\cite{Wiesemann2013RMDP} or policy iteration~\cite{iyengar2005RMDP,Kaufman2013RobustMPI,Ho2018FastBellman} style algorithms. These algorithms can be scaled to infinite-state problems. For instance, Robust Projected Value Iteration (RPVI)~\cite{Tamar2014Scaling} does this for robust value iteration by combining it with linear function approximation.


\section{Static Soft-Robust Framework}
\label{sec:soft-robust}

In this section, we describe the \emph{static} soft robust criterion, discuss its relationship to the percentile criterion, show that optimizing it is NP hard, and compare it with its \emph{dynamic} counterpart. We propose to maximize the \emph{static soft-robust} objective $\srob:\Pi\to\Real$, defined as
\begin{equation} \label{eq:static_soft_robust}
\begin{gathered}
    \max_{\pi\in\Pi} \, \srob(\pi) := \\
 := (1-\lambda) \cdot \underbrace{\E{\rho(\pi,\hat{P})}}_{\text{mean return}} + \lambda \cdot \underbrace{\cvar^\alpha\left[\rho(\pi,\hat{P})\right]}_{\text{robust return}}.
\end{gathered}
\end{equation}
Here, $\cvar^\alpha$ represents the conditional value at risk at level $\alpha$, which is defined for any random variable $Z \sim g$ as~\citep{Rockafellar2000optimizationCVAR}
\begin{equation}\label{eq:cvar}
\cvar^{\alpha}_g[Z] \;:=\; \max_{b\in\real} \, \Big(b - \frac{\E{ \max \{b-Z, 0\} }}{1-\alpha}  \Big)\,.
\end{equation}
The robust return term $\cvar^\alpha\big[\rho(\pi,\hat{P})\big]$ in~\eqref{eq:static_soft_robust} represents the average of the expected returns of the worst $1-\alpha$ fraction of the models. The parameters $\alpha\in[0,1]$ and $\lambda\in[0,1]$ are domain specific and give the decision-maker fine-grained control over the policy's robustness. The parameter $\lambda\in [0,1]$ balances the importance of mean and robust returns. The risk-aversion parameter $\alpha \in [0,1]$ controls the robustness of the return of $\pi$. For example, when $\alpha = 0.95$, the robust return is computed by averaging the returns over the worst $5\%$ of the models.

Comparing the percentile criterion in~\eqref{eq:percentile_var} with the soft-robust criterion in \eqref{eq:static_soft_robust}, one can appreciate how the soft-robust criterion addresses the issues that arise with the percentile criterion: The soft-robust criterion explicitly includes the mean performance (weighted by $1-\lambda$) and $\cvar$ is both, sensitive to the tail of the distribution and convex (e.g.,~\citealt{Shapiro2014StochasticProgramming}).

The following results establish fundamental properties of the soft-robust objective $\rho^S(\pi)$. First, the following proposition justifies the need to consider randomized policies when optimizing this objective.
\begin{prop} \label{prop:randomized_soft_robust}
    There may be no stationary \emph{deterministic} policy $\pi\in\Pi_D$ that attains the optimal objective of the soft-robust optimization problem~\eqref{eq:static_soft_robust}.
\end{prop}
\cref{prop:randomized_soft_robust} follows immediately from {\em Theorem~2} in~\citet{bucholz2020CMFMP} by setting $\lambda =0$ or $\alpha = 0$ in~\eqref{eq:static_soft_robust}. Similar argument shows that history-dependent randomized policies may further outperform stationary ones~\cite{Steimle2018MM}.
    The following proposition establishes the computational complexity of the optimization problem~\eqref{eq:static_soft_robust}.

\begin{prop} \label{prop:soft_robust_static_nphard}
    Computing the optimal (randomized or deterministic) policy of the soft-robust problem~\eqref{eq:static_soft_robust} is NP-hard.
\end{prop}
\cref{prop:soft_robust_static_nphard} follows readily from \emph{Theorem~1} in~\citet{bucholz2019CMDP} by setting $\lambda = 0$ or $ \alpha = 0$.

In the remainder of the section, we argue that our static formulation handles the uncertainty over $\hat{P}$ more accurately than prior dynamic formulations. Model uncertainty has serious consequences in RL. Increasing uncertainty in $\hat{P}$ causes the value function of a non-robust policy to become unrealistically optimistic. This effect, which is driven by always choosing the maximum over uncertain action value estimates, is known as \emph{optimizer's curse}~\cite{smith2006Optimizer}. Unrealistically high value function inherently drives the agent to prefer states with high model uncertainty which is undesirable in robust RL. Double Q-learning~\cite{vanhasselt2015deep} and other methods~\cite{Powell2011Approx,Buckman2021Pessimism} mitigate the curse but do not eliminate it.

We now show that the static soft-robust formulation with $\lambda = 0$ eliminates the optimizer's curse for the mean returns. The dynamic soft-robust formulations~\cite{Xu2012Distributionally,Derman2018Soft-RobustAC} and almost all other RL algorithms suffer from this curse.

To formally define the optimizer's curse~\cite{smith2006Optimizer}, recall that the random variable $P\opt$ represents the true transition probability used to  generate the dataset $\mathcal{D}$. The term \emph{post-decision surprise} refers to the difference $\rho(\bar{\pi}(\data), P\opt) - \bar{\rho}(\data)$ between the true return of $\bar{\pi}(\data)\in\Pi$ and its estimated return $\bar{\rho}(\data) \in \Real$. Note that both the policy $\bar{\pi}(\data)$ and its estimated return $\bar{\rho}(\data)$ depend on the dataset $\data$. If the average post-decision surprise is negative, $\Exp{\mathcal{D},P\opt}{\rho(\bar{\pi}(\mathcal{D}), P\opt) - \bar{\rho}(\mathcal{D})} < 0$, then the algorithm is said to suffer from the
\emph{optimizer's curse}. As described above, consistently optimistic (or biased) return estimates drive the agent to more uncertain states.

We are now ready to show that the static soft-robust formulation is immune to the optimizer's curse. Let $\bar\pi_S(\mathcal{D})$ denote an optimal solution to~\eqref{eq:static_soft_robust} for $\hat{P} = P\opt \ss \mathcal{D}$ and, similarly, let $\bar{\rho}_S(\mathcal{D})$ be the optimal objective value of~\eqref{eq:static_soft_robust}.
\begin{thm} \label{thm:no_bayesian_surprise}
    Optimal solution $\bar\pi_S(\mathcal{D})$ with objective $\bar{\rho}_S(\mathcal{D})$ to \eqref{eq:static_soft_robust} with $\lambda = 0$ has no expected post-decision surprise:
    \begin{equation*}
        \Exp{\mathcal{D},P\opt}{\rho(\bar{\pi}_S(\mathcal{D}), P\opt) - \bar{\rho}_S(\mathcal{D})} = 0 ~.
    \end{equation*}
    Moreover, the expected post-decision surprise is non-negative for any $\lambda \in (0,1]$ and any $\alpha \in [0,1]$.
\end{thm}
The proof of the theorem can be found in~\cref{apps:soft_robust}.

To illustrate the implications of \cref{thm:no_bayesian_surprise}, \cref{fig:optimizers_curse} compares the post-decision surprise of the static soft-robust model with a dynamic model and an empirical method. We use a small MDP with $5$ states and $3$ actions with $P\opt$ sampled from the uniform Dirichlet prior and $|\data| = 100$ drawn from a random policy. The dynamic soft-robust criterion with $\lambda = 0$ solves $\max_{\pi\in\Pi} \rho(\pi,\mathbb E[\hat{P}])$~\cite{Derman2018Soft-RobustAC}. Note that the expectation is inside of the return rather than outside. The empirical method solves for $\max_{\pi\in\Pi} \rho(\pi,\bar{P})$, where $\bar{P}$ are empirical transition probabilities. The results show that the empirical solution consistently suffers from significant negative average post-decision surprise, which is slightly reduced by the dynamic formulation, and eliminated by the static formulation.

\begin{figure}
\centering
\includegraphics[width=0.8\linewidth]{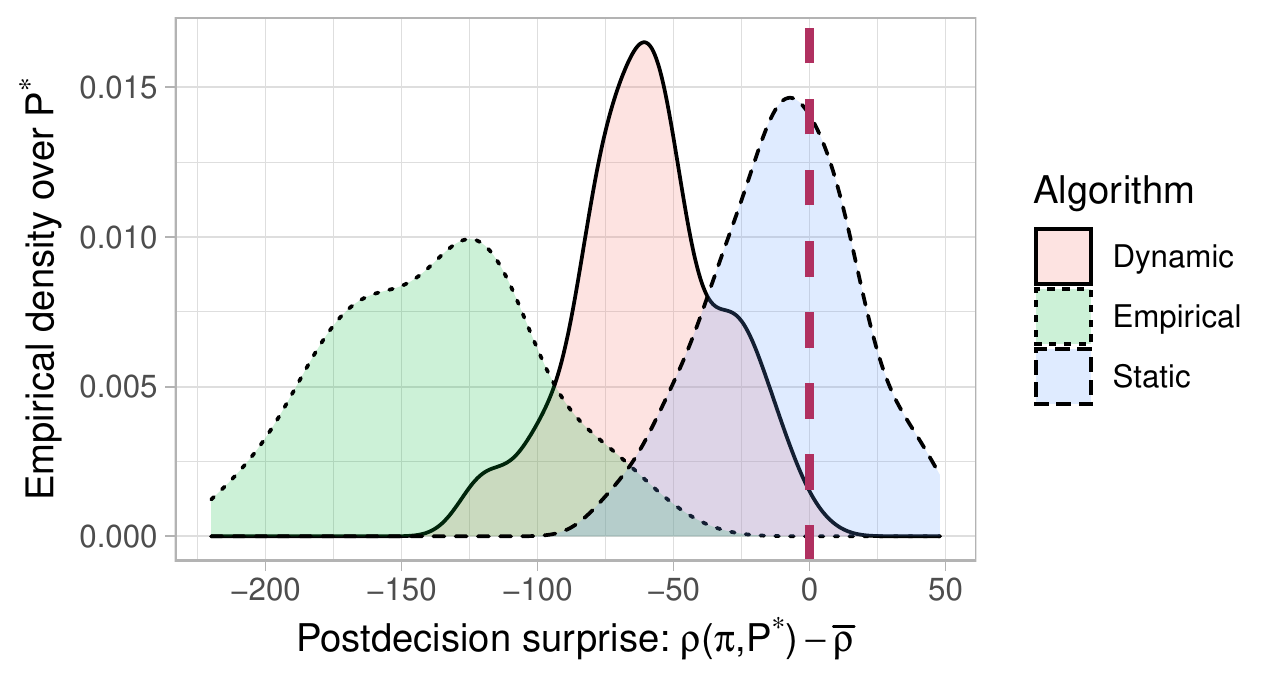}
\caption{Post-decision surprise of policies computed using static and dynamic soft-robust criteria and the empirical model.} \label{fig:optimizers_curse}
\end{figure}


\section{Soft-Robust Optimization}
\label{sec:Soft-Robust-optimization}

In this section, we present two approximate algorithms for maximizing the soft-robust objective. First, we derive a mixed-integer linear program~(MILP) formulation in~\cref{sec:MILP} that can be used to compute an optimal {\em deterministic} policy. Because MILPs do not scale well, we then show how the soft-robust objective can be represented as an RMDP in~\cref{sec:rmdp_formulation}, and solved with linear value function approximation in~\cref{sec:value_iteration}.

We start by stating the following lemma that shows the soft-robust criterion reduces to a worst-case expectation over a certain set of measures over the transition model $\hat{P}$.
\begin{lem} \label{prop:static_soft_robust}
    Define a set $\Xi \subseteq \Delta^{|\Omega|}$ as
    \begin{equation}  \label{eq:xi-set}
        \Xi = \Bigl\{ \xi\in\Delta^{|\Omega|} \ss (1-\lambda)\cdot \Pf \le \xi \le \frac{1 - \alpha + \lambda\alpha }{1-\alpha}\cdot \Pf \Bigr\}.
    \end{equation}
    Then, for each $\pi\in\Pi$, the objective $\rho^S(\pi)$ in~\eqref{eq:static_soft_robust} satisfies
    \begin{equation} \label{eq:non_rect_soft_robust}
        \rho^S(\pi) \;=\; \min_{\xi\in\Xi} \, \mathbb E_{\hat{P} \sim \xi}\big[\rho (\pi, \hat{P})\big]\,.
    \end{equation}
\end{lem}
The proof of~\cref{prop:static_soft_robust}, which we report in~\cref{app:soft_robust_algorithms}, follows by algebraic manipulation from the robust representation of $\cvar$ in \eqref{eq:cvar}. This result allows us to rewrite the {\em static soft-robust} optimization~\eqref{eq:static_soft_robust} as
\begin{equation}
    \label{eq:soft-robust2}
    \max_{\pi\in\Pi} \; \rho^S(\pi) = \max_{\pi\in\Pi} \; \min_{\xi\in\Xi} \; \mathbb E_{\hat{P} \sim \xi}\big[\rho (\pi, \hat{P})\big]\,.
\end{equation}
%


\subsection{Soft-Robust Optimization using MILP}\label{sec:MILP}

\begin{figure*}
    \centering
    \small
    \begin{tcolorbox}[colback=black!0.0!white]
        \begin{equation*} 
            \extrarowheight=1mm
            \begin{array}{>{\displaystyle}c>{\displaystyle}r@{\hspace{1.5mm}}>{\displaystyle}l>{\displaystyle}l}
                \operatorname*{maximize}_{\substack{\pi \in  \{0,1\}^{S \times A}, \, b\in\real,\\
                        u\in\Real^{S\times A\times N}_+,\, y\in \real^N_{+}}}
                &\multicolumn{3}{>{\displaystyle}l}{
                    \lambda \cdot \Big(b- \frac{1}{1-\alpha} \sum_{\omega \in \Omega} y(\omega)\Big) + (1-\lambda) \cdot \sum_{s \in \states} \sum_{a \in \actions} \sum_{\omega \in \Omega} u(s,a,\omega) \sum_{s' \in \states} r(s,a,s') \cdot P^{\omega}(s,a,s')}\\
                \operatorname{subject\,to} &y(\omega) - b \cdot f_{\omega} &\geq\; - \sum_{s \in \states} \sum_{a \in \actions}u(s,a,\omega) \sum_{s' \in \states} P^{\omega}(s,a,s') \cdot r(s,a,s'), &\omega \in \Omega, \\
                &\sum_{a \in \actions} u(s,a,\omega) &=\; \sum_{s' \in \states}\sum_{a' \in \actions} \gamma\cdot u(s',a',\omega)\cdot P^{\omega}(s',a',s) + f_{\omega} \cdot p_0(s), &s \in \states,\, \omega \in \Omega, \\
                &\sum_{a\in\actions} \pi(s,a) &=\; 1, &s \in \states, \\
                &u(s,a,\omega) &\leq\; f_\omega \cdot \pi(s,a) / (1-\gamma),  &s \in \states,\, a \in \actions,\, \omega \in \Omega.
            \end{array}
            \extrarowheight=0mm
        \end{equation*}
    \end{tcolorbox}
    \vspace{-0.3cm}
    \caption{SR-MILP: Mixed-Integer linear program that solves the soft-robust optimization problem~\eqref{eq:soft-robust2}.} \label{fig:MILP}
\end{figure*}

Our MILP formulation of the soft-robust optimization, which we call SR-MILP and present it in~\cref{fig:MILP}, is based on~\eqref{eq:soft-robust2}. The intuitive explanation for this formulation is as follows. One may compute the soft-robust objective by simultaneously solving a series of MDPs with transition functions given by $\hat{P}^\omega$, one for each $\omega \in \Omega$. The variable $u(s,a,\omega) \in \real_+$ in \cref{fig:MILP} represents the occupancy frequency for the state $s$ and action $a$ in the MDP given by $\omega$. The second constraint ensures that $u$ is a valid occupancy frequency. The binary variable $\pi(s,a)\in\{0,1\}$ (deterministic policy), is used to guarantee that the policy is consistent across the MDPs $\hat{P}^\omega$ by enforcing the fourth constraint. The fourth constraint ensures that $u(s,a,\omega) > 0 \Leftrightarrow \pi(s,a) = 1$. Finally, the variables $b$ and $y$ and the first constraint are used to represent the $\cvar$ formulation in~\eqref{eq:cvar}. We are now ready to state the correctness of our formulation in~\cref{prop:milp_static}, whose proof we report in~\cref{app:soft_robust_algorithms}.
\begin{prop} \label{prop:milp_static}
    Any $\pi\opt$ optimal in \cref{fig:MILP} satisfies that $\pi\opt\in \arg\max_{\pi\in\Pi_D} \srob(\pi)$.
\end{prop}
The MILP in \cref{fig:MILP} returns deterministic policies. Although this may seem like a limitation, it actually offers some tangible advantages. In practice, deterministic policies are often preferred over randomized ones, when randomizing between different actions is undesirable. In medical domains, for example, it may be unethical to randomize outside of a medical trial. In other domains, randomization hinders reproducibility and may make it very difficult to evaluate and diagnose the policy once it is deployed.

Although the MILP in \cref{fig:MILP} returns stationary policies, they can still benefit from the static uncertainty assumption in \eqref{eq:static_soft_robust}. To illustrate this point, consider a cancer treatment problem, where the agent has to decide on the amount of chemotherapy to be administered. The fact that a particular state of the patient reveals some information their response to the treatment can be used to make more informed decisions. The dynamic uncertainty model, on the other hand, assumes that the patient model changes throughout the execution and cannot exploit this information.


\subsection{Soft-Robust Optimization using RMDPs} \label{sec:rmdp_formulation}

In this section, we describe how to construct an RMDP that approximates the soft-robust criterion.
First we note that~\citeasnoun{Xu2012Distributionally} showed that optimizing any coherent risk measure of the return (including the soft-robust criterion) is equivalent to solving an RMDP. However, there are three main differences between their results and ours: {\em 1)} To show the equivalence to an RMDP problem,~\citeasnoun{Xu2012Distributionally} assume that each state is visited only once within an episode. This assumption is too strong and does not hold for infinite-horizon MDPs with finite state spaces. Hence, we propose the dynamic soft-robust objective in this section as an approximation to the original soft-robust criterion, and then in~\cref{sec:approximation_error}, bound the error due to this approximation. {\em 2)} We derive explicit RMDP ambiguity sets for soft-robust criterion instead of an abstract representation, as in~\citeasnoun{Xu2012Distributionally}. {\em 3)} We present a scalable algorithm to solve the dynamic soft-robust objective in~\cref{sec:value_iteration}.

Our soft robust reduction to a RMDP proceeds in two steps.

{\em Step~1:} Approximate the soft-robust optimization~\eqref{eq:soft-robust2} as
\begin{equation} \label{eq:dynamic_soft_robust}
\max_{\pi\in\Pi}\; \rho^D(\pi) \;=\; \max_{\pi\in\Pi}\; \min_{\xi\in\Xi} \; \rho \big(\pi, \mathbb {E}_{\hat{P} \sim \xi}[\hat{P}]\big)\,.
\end{equation}
where the superscript $D$ indicates that this is the ambiguity set corresponding to the dynamic soft-robust formulation. Note that, in contrast to $\rho^S(\pi)$ in~\eqref{eq:soft-robust2}, the expectation in~\eqref{eq:dynamic_soft_robust} is inside the return function $\rho$. This approximation is helpful because it can be represented as a non-rectangular RMDP (see~\eqref{eq:rmdp_general}) with the ambiguity set 
\begin{equation} \label{eq:non_rect_ambiguity}
\mathcal{P}^D \subseteq (\Delta^S)^{\states\times\actions}, \; \mathcal{P}^D = \Big\{ \sum_{\omega\in\Omega} \xi_\omega \cdot \hat{P}^\omega \ss \xi\in\Xi \Big\}\,.
\end{equation}

{\em Step~2:} Although solving a non-rectangular RMDP is NP-hard~\cite{Wiesemann2013RMDP}, it can be turned into a tractable rectangular one by a process called \emph{rectangularization} in the context of dynamic risk measures~\cite{Roorda2004CoherentAM,iancu2011TightApprox}. Rectangularization constructs the smallest rectangular set that contains the entire non-rectangular one. Since the rectangular set is a superset of the non-rectangular one, the rectangular robust objective lower-bounds its non-rectangular counterpart.

To formalize the rectangularization procedure, assume that the soft-robust ambiguity set, which we denote by $\mathcal{P}^R$, is S-rectangular, i.e.,\footnote{This ambiguity set decomposition is similar to the one in~\eqref{eq:s-rect}. The mnemonic superscript $R$ in $\mathcal{P}^R$ indicates that it is a S-rectangular ambiguity set.}
\begin{equation}
\label{eq:ambiguity_rect}
\mathcal{P}^R = \bigotimes_{s \in \states} \mathcal{P}_{s}^R, \quad \mathcal{P}_s^R = \Big\{ \sum_{\omega \in \Omega} \xi_{\omega} \cdot \hat{P}_s^{\omega} \ss \xi \in \Xi \Big\}.
\end{equation}
Here $\Xi$ is defined in~\eqref{eq:xi-set} and $\hat{P}_s^{\omega} \in
(\Delta^S)^A$ is a value of the posterior sample $\hat{P}^{\omega}$ in state $s$. Recall from Section~\ref{sec:motivation} that the S-rectangular RMDPs can be solved efficiently~\cite{Ho2018FastBellman}. 
Finally, the following optimization problem defines the S-rectangular objective $\rrob : \Pi \to \Real$:
\begin{equation} \label{eq:objective_rectangular}
\rrob(\pi) \;=\; \min_{P \in \mathcal{P}^R} \; \rho \left(\pi, P \right)\,.
\end{equation}
The following proposition (proof in \cref{app:soft_robust_algorithms}) shows how the non-rectangular ambiguity set $\mathcal{P}^D$ in~\eqref{eq:non_rect_ambiguity} and its corresponding return $\rho^D$, defined in~\eqref{eq:dynamic_soft_robust}, are related to the S-rectangular ambiguity set $\mathcal{P}^R$ and return $\rrob$.
\begin{prop}  \label{thm:s_rect_soft_robust}
The ambiguity sets $\mathcal{P}^D$ and $\mathcal{P}^R$ satisfy $\mathcal{P}^D \subseteq \mathcal{P}^R$, and their corresponding returns $\rho^D$ and $\rrob$ satisfy $\rrob(\pi) \le \drob(\pi)$, for each $\pi\in\Pi$.
\end{prop}
Proposition~\ref{thm:s_rect_soft_robust} shows two important results. {\em First}, the S-rectangular ambiguity set $\mathcal{P}^R$ contains the non-rectangular ambiguity set $\mathcal{P}^D$ (rectangularization procedure). {\em Second}, the S-rectangular objective $\rrob$, which can be tractably computed by solving the resulting S-rectangular RMDP, is a lower-bound of the dynamic soft-robust objective $\rho^D$.

The optimal value function $v^R\in\real^S$ of the S-rectangular RMDP defined by $\mathcal{P}^R$ satisfies the robust Bellman optimality equation $v^R = \Bell_{\mathcal{P}^{\sr}} v^R$ and can be approximated using the standard value iteration algorithm (see Section~\ref{sec:motivation}).
For any $v\in\real^S$ and $s\in\states$, the Bellman optimality operator $(\Bell_{\mathcal{P}^{\sr}} v)(s)$ can be computed by solving the following linear program with $z_{s,a} = r_{s,a} + \gamma  \cdot v$:
\begin{equation} \label{eq:linear_program}
\begin{aligned}
&\max_{\substack{d \in \Delta^A,\,b\in\real\\y\in\real_+^{|\Omega|}}} \quad (1- \lambda) \sum_{\substack{a \in \actions\\\omega\in\Omega}} d_a  \Pf_\omega (\hat{P}_{s,a}^{\omega})\tr z_{s,a}  \\
&\qquad\qquad\qquad + \lambda \Big(b - \frac{1}{1-\alpha} \sum_{\omega\in\Omega} \Pf_\omega \cdot y_\omega \Big) \\
&\st \qquad y_\omega \geq b- \sum_{a \in \actions} d_a (\hat{P}_{s,a}^{\omega})\tr z_{s,a}, \quad \omega \in \Omega\,.
\end{aligned}
\end{equation}
\begin{prop} \label{prop:Bellman_update}
For any $v\in\real^S$ and $s\in\states$, the optimal value of the objective function in~\eqref{eq:linear_program} is equal to $(\Bell_{\mathcal{P}^{\sr}} v)(s)$.
\end{prop}
The correctness of~\cref{prop:Bellman_update} follows from algebraic manipulation of~\eqref{eq:ambiguity_rect} and is provided in~\cref{app:soft_robust_algorithms}. \newline


\subsection{Projected Soft-Robust Value Iteration} \label{sec:value_iteration}

We now present our soft-robust value iteration (SRVI) algorithm that we use to (approximately) solve the soft-robust S-rectangular RMDP defined in Section~\ref{sec:rmdp_formulation}. SRVI, whose pseudo-code is shown in \cref{alg:norbu_s}, generalizes the Robust Projected Value Iteration (RPVI) algorithm~\cite{Tamar2014Scaling} to soft-robust S-rectangular RMDPs.

We use the linear approximation $v(s) = \phi(s) \tr \w$ for the soft-robust value function $v \in \real^{S}$, where $\phi(\cdot) \in \real^{l},\;l \ll S$ is an $l$-dimensional feature vector and $\w \in \real^{l}$ is a weight vector. We denote by $\Phi \in \real^{M \times l}$ the sample feature matrix of $\phi$ after observing $M$ samples, and by $h \in \Delta^{S}$ the steady state distribution of any given policy $\pi \in \Pi$ over states $s \in \mathcal{S}$. Further, let $\sigma_{\Phi^{\top} \w}: \real^S\to \real^S$ be the function obtained by applying the S-rectangular soft-robust Bellman optimality operator to a value function $v = \Phi \w$, i.e.,
\begin{equation} \label{eq:inner_optimization}
\begin{gathered}
\sigma_{\Phi^\top\w}(s) = (\Bell_{\mathcal{P}^{\sr}} v )(s) = (\Bell_{\mathcal{P}^{\sr}} (\Phi^{\top} \w) )(s).
\end{gathered}
\end{equation}
Then $\Sigma_{\Phi^{\top}\w}$ represents the vector of the soft-robust Bellman optimality values for matrix $\Phi$: $\{\sigma_{\Phi^{\top}\w}(s_t)\}_{t=1}^M$.
Finally, we denote by $\Psi$ the projection operator onto the subspace $\Phi$ w.r.t.~the $h$-weighted Euclidean norm.

\begin{algorithm}
\SetAlgoLined
\KwIn{confidence $\alpha$, risk factor $\lambda$, distribution $\Pf$}
    \KwOut{soft-robust value function $v$
}
    {\em Initialize:} weight vector $w_0$; $\;$ counter $k\gets1$\;
    Sample $N$ parametric models $\{\hat{P}_{\theta}^{\omega_i}\}_{i=1}^N$ from $\Pf$ \;
    Compute mean $\bar{P}_{\theta} = \mathbb{E}[\hat{P}_{\theta}]$ from  $\{\hat{P}_{\theta}^{\omega_i}\}_{i=1}^N$\;
    \Repeat{$\;\norm{\Phi_k\tr \w_{k} - \Phi_k\tr \w_{k-1}}_{\infty} \le \epsilon$}
    {
        Simulate episodes following $\bar{P}_{\theta}$ and policy from~\eqref{eq:inner_optimization} to get
        samples $\mathcal{D}_k$ and $\Phi_k$\;
        Compute $w_k$ from~\eqref{eq:sr_linear} using $\Phi = \Phi_k$\;
        $k\gets k+1$ \;
    }
    \Return{$\;v=\Phi_k \w_k$}
    \caption{Soft-Robust Value Iteration~(SRVI)}
    \label{alg:norbu_s}
\end{algorithm}

SRVI approximates $\pi_{R}\in\arg\max_{\pi\in\Pi} \rho^R(\pi)$ (see Eq.~\ref{eq:objective_rectangular}) by iteratively solving the projected soft-robust Bellman optimality equation $v =\Psi \, \Bell_{\mathcal{P}^{\sr}} \, v$. In each iteration $k$, we first simulate episodes using the mean transition probability model $\bar{P}_{\theta}$ and construct the dataset $\mathcal D_k$ of size $M$ to approximately represent the stationary state distribution induced by the current policy.
Then, we update the weight vector $\w_k$ using the reconstructed data as
\begin{equation} \label{eq:sr_linear}
\w_{k} = (\Phi\tr H \Phi)^{-1} (\Phi\tr H P \, \Sigma_{\Phi^{\top}\w_{k-1}}),
\end{equation}
where $H = \diag(h)$. Since it is impossible to exactly compute the terms in~\eqref{eq:sr_linear}, we approximate them using the Sample Average Approximation (SAA) as
\begin{equation} \label{eq:rpvi_updates}
\begin{aligned}
    \Phi\tr H \Phi &\;\sim\; \frac{1}{M} \sum_{t=1}^{M} \phi(s_t) \phi(s_t)\tr \;, \\
    \Phi\tr H P_{\theta} \Sigma_{\Phi^{\top}\w} &\;\sim\; \frac{1}{M} \sum_{t=1}^{M} \sigma_{\Phi^{\top}\w}(s_t).
    \end{aligned}
\end{equation}
The optimization problem in~\eqref{eq:inner_optimization} can be solved by formulating it as a linear program in~\eqref{eq:linear_program}, and using the SAA method to approximate the value function $v$. We repeatedly update the weight vector $\w$ using \eqref{eq:sr_linear} until the soft-robust value function $\Phi\tr \w$ converges to the unique projected fixed-point of $\Bell_{\mathcal{P}^{\sr}}$. Given the optimal weight vector $\w\opt$, the optimal policy for any state $s$ can be then computed by solving~\eqref{eq:inner_optimization}.


\section{RMDP Approximation Error}
\label{sec:approximation_error}

In this section, we derive approximation error bounds on the RMDP formulation described in \cref{sec:rmdp_formulation}. These bounds provide insight into the possible directions for further improvement of the formulation. The error introduced by resorting to the RMDP formulation depends on two main factors: 1) how the model uncertainty impacts the occupancy frequency, and 2) whether there exists some ordering of $\omega_1,\ldots,\omega_N \in \Omega$ such that the model $P^{\omega_i}$ is approximately better than $P^{\omega_{i+1}}$ consistently across the states.

The proof of the approximation error proceeds in the same two steps as in \cref{sec:rmdp_formulation}. The error introduced in the first step depends on the state occupancy frequency $h_{\pi}^{\omega} \in \Real^S$ for each $\pi\in\Pi$ and $\omega\in\Omega$ defined as
\begin{equation}\label{eq:occupancy_frequency}
h_{\pi}^{\omega} \;=\; \big( \eye - \gamma \cdot \hat{P}^{\omega\top}_\pi \big)^{-1} p_0~.
\end{equation}
We get the following bound on the \emph{first step}'s error.
\begin{thm}\label{lem:static_dynamic_difference}
The difference between static and dynamic returns is bounded for each $\pi\in\Pi$ as
\[
\left\lvert \drob(\pi) - \srob(\pi) \right\rvert \;\le\; \frac{\gamma \cdot r_{\max}}{1-\gamma} \cdot \epsilon_1(\pi)~,
\]
where
$
\epsilon_1(\pi) \;=\; \max_{\omega_1,\omega_2\in\Omega} \norm{h_{\pi}^{\omega_1} - h_{\pi}^{\omega_2}}_1~.
$
\end{thm}
We provide the proof of the theorem in \cref{app:bounds_proofs}. Its main idea is to bound the nonlinearity of $c: \xi \mapsto \rho(\pi, \mathbb{E}_{\hat{P} \sim \xi}[\hat{P}])$. In particular, when $\epsilon_1 = 0$ then $c$ is linear and
\[ \srob(\pi) = \min_{\xi\in\Xi} \mathbb{E}[\rho(\pi, \hat{P})]  = \min_{\xi\in\Xi} \rho(\pi, \mathbb{E}[\hat{P}]) = \drob(\pi)\,.  \]

The following lemma bounds the error that arises due to the rectangularization in the \emph{second step} of the approximation.
\begin{lem} \label{thm:error_bound_rectangularization}
Suppose that $\pi_{D}\opt \in \arg\max_{\pi\in\Pi} \drob(\pi)$ and $\pi_R\opt \in \arg\max_{\pi\in\Pi} \rrob(\pi)$. Then:
\[
\drob(\pi\opt_D) - \drob(\pi\opt_R) \;\le\; \frac{1}{1-\gamma} \cdot \epsilon_2~,
\]
where $\epsilon_2 = \max_{s\in\states,a\in\actions} \min_{\xi\in\Xi} \delta(s,a,\xi)$,  $\delta(s,a,\xi) = \sum_{\omega\in\Omega} \xi_\omega\cdot (\hat{P}_{s,a}^{\omega})\tr (r_{s,a} + \gamma \cdot v\opt_D) - (v\opt_D)_{s}$,
 $v\opt_D\in\Real^S$ is the value function of $\pi_D\opt$.
\end{lem}
The proof of the lemma is reported in \cref{app:bounds_proofs}.

The value $\delta(s,a,\xi)$ in \cref{thm:error_bound_rectangularization} is the difference between the rectangular robust Bellman value in the state-action pair $s,a$ and the non-rectangular RMDP value of $s$. It can be readily seen that $\epsilon_2$ vanishes when there exists an ordering of elements $\omega_1, \omega_2, \ldots$ in $\Omega$ such that $\hat{P}_{s,a}(\omega_i)\tr (r_{s,a} + \gamma \cdot \hat{v}) \ge \hat{P}_{s,a}(\omega_{j})\tr (r_{s,a} + \gamma \cdot \hat{v})$, for $i < j$ and for all states and actions. This is because if the condition holds then the set $\arg \min_{\xi\in\Xi} \delta(s,a,\xi)$ is constant across states and actions and equals to $\arg\min_{\xi\in\Xi} \rho(\pi,\mathbb{E}_{\hat{P}\sim\xi}[\hat{P}])$.

We can now bound the overall RMDP approximation error.
\begin{cor} \label{thm:lower_bound}
The soft-robust return $\srob$ of $\pi\opt_R \in \arg\max_{\pi\in\Pi} \rrob(\pi)$ computed by \cref{alg:norbu_s} satisfies that
\[
\srob(\pi\opt_S) - \srob(\pi\opt_R) \;\le\;  \frac{1}{1-\gamma} \left(2\cdot \gamma \cdot\epsilon_1\cdot r_{\max} + \epsilon_2 \right),
\]
where $\epsilon_1 = \max_{\pi\in\Pi} \epsilon_1(\pi)$, and $\epsilon_1(\pi)$ and $\epsilon_2$ are defined as in \cref{lem:static_dynamic_difference,thm:error_bound_rectangularization} respectively, and $\pi_{S}\opt \in \arg\max_{\pi\in\Pi} \srob(\pi)$  .
\end{cor}
The proof of the corollary is reported in \cref{app:bounds_proofs}.


\section{Experimental Evaluation}
\label{sec:experiments}

In this section, we present two case studies to demonstrate the performance of the soft-robust criterion.
We compare soft-robust algorithms with related baseline algorithms in terms of the mean and robust performance of over the posterior distribution $\Pf$ inferred from the fixed dataset $\mathcal{D}$. Please refer to \cref{app:experimental_results} for a more detailed detailed description.


\subsection{Integrated Pest Control Problem}

The domain represents a simplified integrated pest control problem. The decision-maker must decide which, if any, pesticide to use during the growing season. The state represents the pest population, and action determines whether a pesticide is used. Exponential pest growth dynamics drive the transition model and the rewards measure the net profit of the yields less the pesticide costs. The corresponding MDP consists of $51$ states, each represents the current pest population as determined by trapping ($0$ means no pest population). Each one of $5$ actions available prescribes the use of an increasingly potent pesticide. The true transition probabilities use a logistic population growth model as described in~\citet{Tirizoni2018PolicyC} and the discount rate is $\gamma = 0.9$.

\begin{figure}
    \centering
    \includegraphics[width=0.8\linewidth]{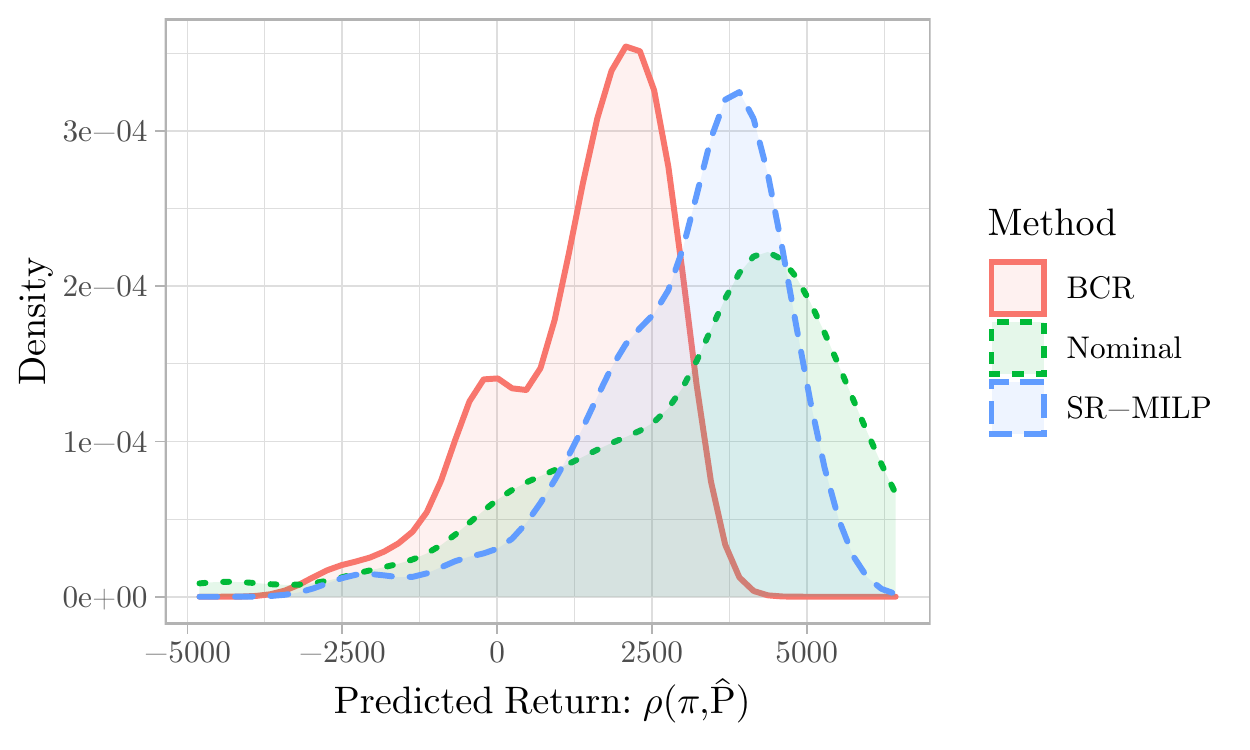}
    \vspace{-0.3cm}
    \caption{Comparison of the densities of $\rho(\pi, \hat{P})$ for three policies.} \label{fig:density_comparison}
\end{figure}

To compute the posterior distribution over $\hat{P}$, we gather $300$ state-action transition samples from a single episode. Using these transition samples, we fit an exponential population model~\cite{Kery2012Bayesian} and sample $100$ posterior samples using MCMC. We use these samples to formulate and solve the MILP in \cref{fig:MILP} and to run \cref{alg:norbu_s}. We use confidence $\alpha = 0.7$ for both the percentile criterion and soft-robust objective for the evaluation. We also use $\lambda = 0.5$ for the soft-robust objective.

\cref{fig:density_comparison} compares the return distribution of the soft-robust MILP policy with the robust BCR solution~\cite{russel2019BeyondCR}, and the nominal policy, which solves the expected transition model $\mathbb{E}[\hat{P}]$. Although the nominal policy achieves the highest mean return, it has a significant probability of incurring loss over $\$5,000$. The BCR policy that targets the percentile criterion improves robustness, but still ignores the fat tail and degrades the mean return.

\Cref{fig:tradeoff_population} shows the trade-off between mean and worst-case performance for several robust methods for different choices of $\lambda \in [0,1]$ as indicated by the floating labels. We compare the optimal MILP policy in \cref{fig:MILP} (SR-MILP) and \cref{alg:norbu_s} (SRVI) with BCR and RSVF~\cite{russel2019BeyondCR}. Note that RSVF and BCR optimize the percentile criterion, which has no inherent notion of the trade-off between robust and mean performance. We simulate the effect of $\lambda$ by simply shrinking the ambiguity sets in the RMDP formulations (multiplying the budget by $\lambda$). Our soft-robust algorithms outperform earlier methods and trade-off well between the mean and robust return with changing $\lambda$.

\begin{figure}
	\centering
	\includegraphics[width=0.85\linewidth]{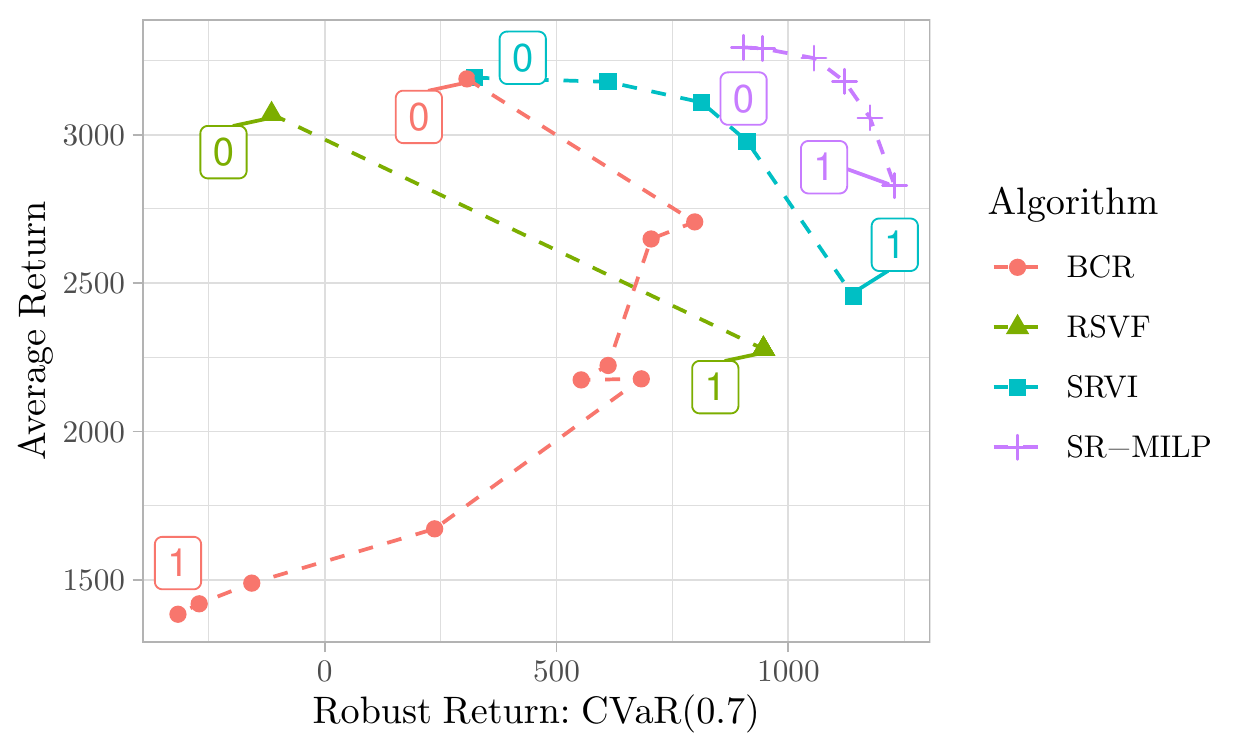}
	\vspace{-0.3cm}
	\caption{Comparison between the trade-offs of several algorithms as parameterized by $\lambda$ as indicated by the overlay label.}
	\label{fig:tradeoff_population}
\end{figure}


\subsection{Cancer Growth Simulator}

The cancer simulator models the growth of tumors in cancer patients. The state is a 4-dimensional vector that captures the dynamics of the tumor's growth. The monthly binary action determines whether to administer chemotherapy to the patient~\cite{gottesman2020interpretable,Ribba2012ATG}. The discount factor $\gamma$ is set to 0.9.

This experiment compares dynamic soft-robust criterion with the dynamic soft-robust objective~\cite{Derman2018Soft-RobustAC} and the robust objective in~\citet{Mankowitz2020RobustRL}.
We combine these robust objectives with the Soft Actor-Critic~(SAC) algorithm~\cite{haarnoja18bSA} to obtain two algorithms which we call Soft-Robust SAC~(SR-SAC) and Robust SAC~(R-SAC). We use SAC instead of robust Q-learning~\cite{hester2017DQ,lillicrap2019DDPG} or other actor-critic~\cite{Konda2000Actor-Critic} algorithms because it has been observed to be more stable. For each algorithm, we train 5 separate agents using 100 samples of $\hat{P}\sim\Pf$ and evaluate the mean and robust return of the computed policies using a separate set of 50 samples of $\hat{P}\sim\Pf$. The robust return is computed using $\cvar$ with $\alpha = 0.9$.

\cref{fig:cvar_cancer} compares the mean and robust performance of SR-SAC and R-SAC with SRVI for $\lambda \in \{0.25,0.75,1.0\}$. SRVI outperforms SR-SAC and R-SAC in mean and robust performance for appropriately chosen $\lambda$. This behavior is expected since SR-SAC ignores the robust return and R-AC ignores the mean return. Focusing on the returns on the training set, SRVI's robust performance improves with an increasing $\lambda$ and the mean performance improved with a decreasing $\lambda$. This expected trend, however, does not quite hold for the test set because of the generalization error. We leave studying the generalization issue for future work.

\begin{figure}
	\centering
	\includegraphics[width=0.8\linewidth]{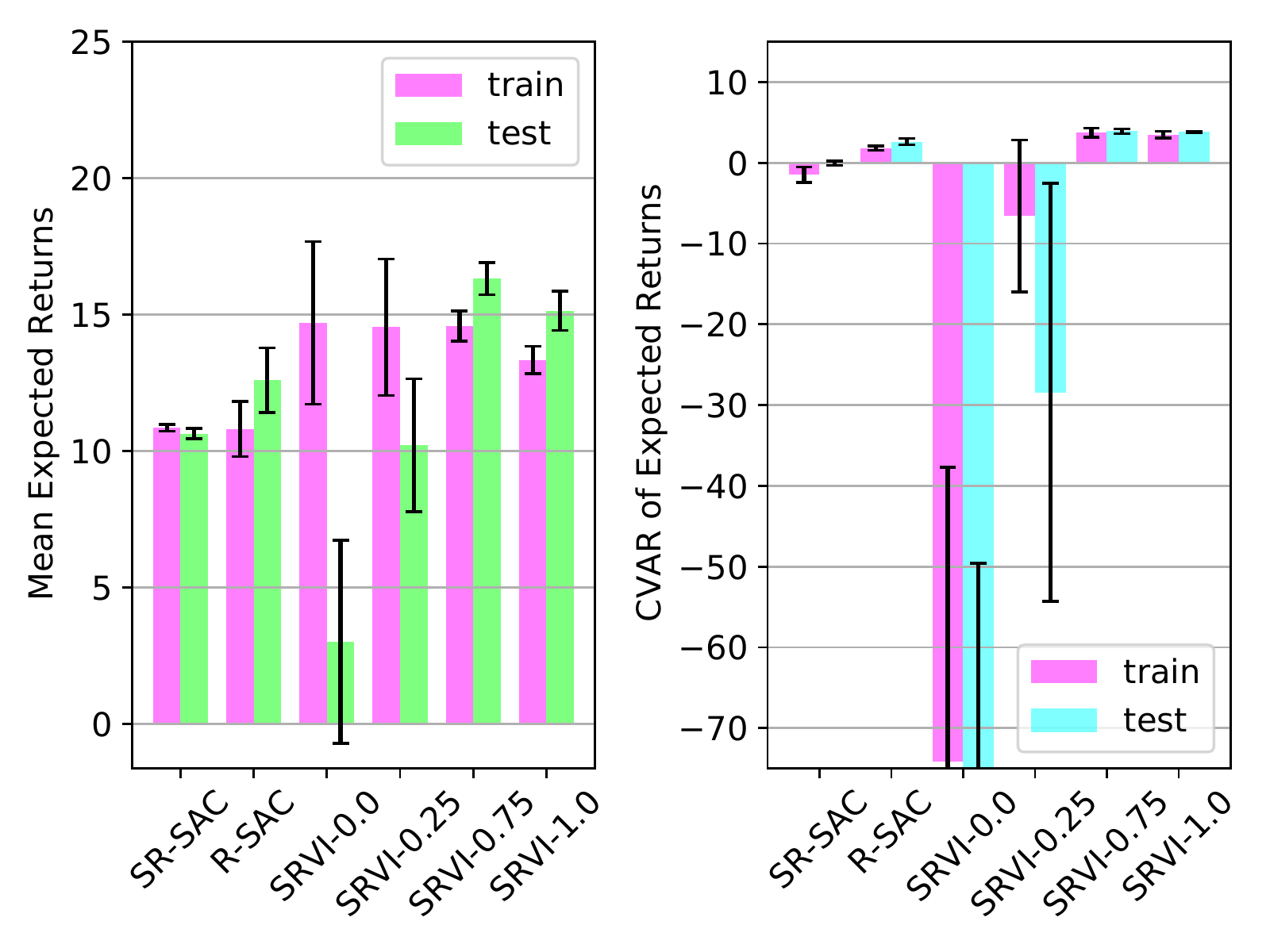}
	\vspace{-0.5cm}
	\caption{Mean and Robust performance of SRVI, SR-SAC, and R-SAC in Cancer Environment.}
\label{fig:cvar_cancer}
\end{figure}


\section{Conclusion}

We proposed a new static soft-robust framework that can balance expected and robust performance in reinforcement learning and handle heavy tail risks. We that the soft-robust objective can be formulated and solved as a MILP for deterministic policies. To scale to larger problems, we propose a new specific RMDP formulation which we combine with linear value function approximation. Finally, we analyze the approximation error of the RMDP formulation and evaluate the algorithms on two domains.

\bibliography{bibliography}
\bibliographystyle{icml2021}

\appendix
\onecolumn

\section{SA-rectangular Soft-Robust MDPs} \label{app:sa_rectangular}

Previously, we described the S-rectangular soft-robust MDPs as a tractable approach for solving the dynamic soft-robust objective. In this section, we extend our results to soft-robust MDPs with SA-rectangular ambiguity set. The SA-rectangular ambiguity set is the simplest class of ambiguity sets and is defined as~\citep{Nilim2005RobustControl,Wiesemann2013RMDP}
\begin{equation*}
    \PP^{SA} = \left\{ p \in (\Delta^S)^{S \times A} \ss p_{s,a} \in \mathcal{P}^{SA}_{s,a}, \; \forall s \in \states, \; \forall a \in \actions \right\},
\end{equation*}
for some $\mathcal{P}^{SA}_{s,a} \subseteq \Delta^S, s\in\states, a\in\actions$. The SA-rectangular ambiguity sets assume that the transition models corresponding to each state-action pair are independent. The robust Bellman optimality operator $\mathcal{T}_{\mathcal{P}^{SA}} : \real^S \to \real^S$ for SA-rectangular ambiguity set $\mathcal{P}^{SA}$ is defined as~\cite{iyengar2005RMDP,Wiesemann2013RMDP}
\begin{equation} \label{eq:SA-rect}
    (\Bell_{\mathcal{P}^{SA}} \, v)(s) = \max_{a\in\actions} \min_{P \in \mathcal{P}^{SA}_{s,a} } \; P_{s,a}\tr \left(r_{s,a}  + \gamma \cdot v \right).
\end{equation}
In SA-rectangular MDPs, the optimal policies are deterministic~\cite{Wiesemann2013RMDP} and the optimal value function $v\opt\in\real^S$ is the unique fixed-point of the robust Bellman optimality operator $\Bell_{\mathcal{P}^{SA}} : \real^S \to \real^S$

The SA-rectangular MDP corresponding to the dynamic soft-robust objective can be derived following a procedure similar to the 2-steps procedure in~\cref{sec:rmdp_formulation}. The only difference in the procedure is that, instead of assuming that the ambiguity set $\mathcal{P}^{R}$ is S-rectangular in Step 2, we would assume that the soft-robust ambiguity set $\mathcal{P}^{R}$ is SA-rectangular, i.e.,
\begin{small}
    \begin{equation}
        \label{eq:ambiguity_rect_sa}
        \mathcal{P}^R = \bigotimes_{s \in \states} \mathcal{P}_{s,a}^R, \;\;\; \text{where } \; \mathcal{P}_{s,a}^R = \Big\{ \sum_{\omega \in \Omega} \xi_{\omega} \cdot \hat{P}_{s,a}^{\omega} \ss \xi \in \Xi \Big\}.
    \end{equation}
\end{small}
To differentiate from the S-rectangular soft-robust ambiguity sets denoted by $\mathcal{P}^{R}$, we denote the soft-robust ambiguity sets with the SA-rectangularity assumption by $\mathcal{P}^{RA}$.
The resulting SA-rectangular soft-robust MDP objective which we denote by $\sarob : \Pi \to \Real$ can be expressed as:
\begin{equation} \label{eq:objective_sarectangular}
    \sarob(\pi) \;=\; \min_{P \in \mathcal{P}^{RA}} \; \rho \left(\pi, P \right)\,.
\end{equation}
The optimal value function $v^{RA}\in\real^S$ of the SA-rectangular soft-robust MDP satisfies the robust Bellman optimality equation $v^{RA} = \Bell_{\mathcal{P}^{RA}} v^{RA}$ where $\Bell_{\mathcal{P}^{RA}}$ is the robust Bellman optimality operator in~\cref{eq:SA-rect} for ambiguity set $\mathcal{P}^{RA}$. We term the operator $\Bell_{\mathcal{P}^{RA}}$ as the SA-rectangular soft-robust Bellman optimality operator. Note that like all robust Bellman optimality operators defined for SA-rectangular ambiguity sets, the SA-rectangular soft-robust Bellman optimality operator $\Bell_{\mathcal{P}^{RA}}$ as well, always results in deterministic optimal policies \cite{Wiesemann2013RMDP}. Hence, contrary to the S-rectangular soft-robust Bellman optimality operator $\Bell_{\mathcal{P}^{R}}$, we need not solve the linear program in~\cref{eq:linear_program} for computing it. Instead it is sufficient to evaluate the convex combination of $\cvar$ and mean of expected returns for each action independently and then choose the soft-robust returns corresponding to the best action as shown in~\eqref{eq:inner_optimization_sa}. The $\cvar$ measure can be computed in quasi-linear time~\cite{Acerbi2002expected} and, therefore, the overall complexity of computing $(\Bell_{\mathcal{P}^{RA}} v)(s)$ is $\mathcal{O}(SAN \log N)$. The linear program for the S-rectangular soft-robust Bellman optimality operator on the other hand, takes polynomial time to solve.

To solve the SA-rectangular soft-robust MDP, we can use the Soft-Robust Value Iteration algorithm discussed in \cref{sec:value_iteration} with one key difference. Instead of using the S-rectangular soft-robust Bellman optimality operator $\Bell_{\mathcal{P}^{R}}$ in ~\eqref{eq:inner_optimization} to derive the policy and construct $\Sigma_{\Phi^{\top}\w}$, we would use the SA-rectangular soft-robust Bellman optimality operator $\Bell_{\mathcal{P}^{RA}} v^{RA}$, i.e.,
\begin{equation} \label{eq:inner_optimization_sa}
    \begin{gathered}
        \sigma_{\Phi^\top\w}(s) = (\Bell_{\mathcal{P}^{RA}} v )(s) = (\Bell_{\mathcal{P}^{RA}} (\Phi^{\top} \w) )(s) = \max_{a\in\actions} \min_{P_{s,a} \in \mathcal{P}^{RA}_{s,a} } \;  P_{s,a}\tr \left(r_{s,a}  + \gamma \cdot v \right)
    \end{gathered}
\end{equation}
In this setting, $\Sigma_{\Phi^{\top}\w}$ is constructed as a vector of the SA-rectangular soft-robust Bellman optimality values for matrix $\Phi$: $\{\sigma_{\Phi^{\top}\w}(s_t)\}_{t=1}^M$.

\section{Auxiliary Results}

The following lemma will be useful when bounding the RMDP approximation of the soft-robust objective.
\begin{lem} \label{lem:matrix_norm}
    The vector-induced norms for a stochastic matrix $P\in\real^{S\times S}$ satisfy that \[\norm{P}_\infty = \norm{P\tr}_1 = 1~. \]
\end{lem}
\begin{proof}
    Let $\mathcal{L}_1 = \{x \in \real^S \ss \norm{x}_1 = 1 \}$ be the $L_1$ ball and let $\mathcal{L}_\infty = \{x \in \real^S \ss \norm{x}_\infty = 1 \}$ be the $L_\infty$ ball.  Then, using basic linear algebra, definition of induced matrix norms in steps $(a)$, and the duality of the vector $L_1$ and $L_\infty$ norm in step $(b)$, we can establish the desired result as follows:
    \[
    \norm{P\tr}_1 \stackrel{(a)}{=} \max_{x\in\mathcal{L}_1}  \,\norm{P\tr x}_1
    = \max_{x \in\mathcal{L}_1} \max_{y\in\mathcal{L}_\infty}  y\tr P\tr x
    \stackrel{(a)}{=} \max_{y\in\mathcal{L}_\infty}  \norm{P y}_\infty = \norm{P}_\infty.
    \]
    The result follows because, as it is well-known, $\norm{P}_\infty = 1$ for any stochastic matrix $P$.
\end{proof}

The following generic lemma establishes the bounds on the error between a maximizer of a function and a maximizer of an approximation of that function.
\begin{lem} \label{lem:difference_any}
    Let $x\opt \in \arg\max_{x\in\mathcal{X}} f(x)$ and $\tilde{x}\opt \in \arg\max_{x\in\mathcal{X}} \tilde{f}(x)$ be the maximizers of some function $f : \mathcal{X} \to \real$ and its approximation $\tilde{f} : \mathcal{X} \to \real$ respectively. Then the optimality gap of $\tilde{x}\opt$ in non-negative and bounded by:

    \begin{equation} \label{eq:difference_any}
        f(x\opt) - f(\tilde{x}\opt) \;\le\; \lvert f(x\opt) - \tilde{f}(x\opt) \rvert + \lvert f(\tilde{x}\opt) - \tilde{f}(\tilde{x}\opt) \rvert \le 2 \cdot \max_{x\in\mathcal{X}} \, \lvert f(x) - \tilde{f}(x) \rvert \,.
    \end{equation}
    Moreover, when $\tilde{f}(x) \le f(x)$ for all $x\in\mathcal{X}$, then the optimality gap of $\tilde{x}\opt$ reduces to:
    \begin{equation} \label{eq:difference_bounded}
        f(x\opt) - f(\tilde{x}\opt) \;\le\; f(x\opt) - \tilde{f}(x\opt)\,.
    \end{equation}
\end{lem}
\begin{proof}
    First, the following basic inequality follows by algebraic manipulation as:
    \begin{subequations}
        \begin{align}
            \nonumber
            f(x\opt) - f(\tilde{x}\opt) &= f(x\opt) - f(\tilde{x}\opt) + \tilde{f}(\tilde{x}\opt) - \tilde{f}(\tilde{x}\opt) &&\text{Add $0$} \\
            \nonumber
            &= f(x\opt) - \tilde{f}(\tilde{x}\opt) + \Bigl(\tilde{f}(\tilde{x}\opt) - f(\tilde{x}\opt) \Bigr) && \text{Rearrange} \\
            \label{eq:proof_diff_any_bound}
            &\le f(x\opt) - \tilde{f}(x\opt) + \Bigl(\tilde{f}(\tilde{x}\opt) - f(\tilde{x}\opt) \Bigr) && \text{Optimality of $\tilde{x}\opt$} \\
            \label{eq:proof_diff_any_bound_abs}
            &\le \left\lvert f(x\opt) - \tilde{f}(x\opt) \right\rvert + \left\lvert \tilde{f}(\tilde{x}\opt) - f(\tilde{x}\opt) \right\rvert~.
        \end{align}
    \end{subequations}
    Then, the inequality~\eqref{eq:difference_any} follows from~\eqref{eq:proof_diff_any_bound_abs} because $x\opt \in \mathcal{X}$ and $\tilde{x}\opt \in \mathcal{X}$ and therefore
    \begin{align*}
        \left\lvert f(x\opt) - \tilde{f}(x\opt) \right\rvert &\le \max_{x\in\mathcal{X}} \, \lvert f(x) - \tilde{f}(x) \rvert \\
        \left\lvert \tilde{f}(\tilde{x}\opt) - f(\tilde{x}\opt) \right\rvert &\le \max_{x\in\mathcal{X}} \, \lvert f(x) - \tilde{f}(x) \rvert~.
    \end{align*}
    The inequality~\eqref{eq:difference_bounded} follows from~\eqref{eq:proof_diff_any_bound} because $\tilde{f}(x) \le f(x)$ for all $x\in\mathcal{X}$ and therefore
    \begin{align*}
        f(x\opt) - f(\tilde{x}\opt) &\le f(x\opt) - \tilde{f}(x\opt) + \Bigl(\tilde{f}(\tilde{x}\opt) - f(\tilde{x}\opt) \Bigr) \\
        &\le f(x\opt) - \tilde{f}(x\opt) && \text{Because $\tilde{f}(\tilde{x}\opt) - f(\tilde{x}\opt) \le 0$} ~.
    \end{align*}

\end{proof}

\section{Proofs: \cref{sec:soft-robust}} \label{apps:soft_robust}

\begin{proof}[Proof of \cref{thm:no_bayesian_surprise}]
    We omit the dependence of $\bar{\pi}_S$ and $\bar{\rho}_S$ on $\mathcal{D}$ to reduce clutter. No post-decision regret for $\lambda = 0$ follows by conditioning on the dataset $\mathcal{D}$ as:
    \begin{align*}
        \Ex{\mathcal{D},P\opt}{\rho(\bar{\pi}_S, P\opt) - \bar{\rho}_S }	 &= \Ex{\mathcal{D}}{\Ex{P\opt}{\rho(\bar{\pi}_S, P\opt) - \bar{\rho}_S \ss \mathcal{D} }} & \text{Property of $\mathbb{E}$}\\
        &= \Ex{\mathcal{D}}{\Ex{P\opt}{\E{\rho(\bar{\pi}_S, P\opt) \ss \mathcal{D} } - \bar{\rho}_S \ss D}} & \text{Property of $\mathbb{E}$} \\
        &= \Ex{\mathcal{D}}{\Ex{P\opt}{\rho(\bar{\pi}_S, P\opt) \ss \mathcal{D} } - \bar{\rho}_S} & \text{$\bar\rho_S$ constant for $P\opt$} \\
        &= \Ex{\mathcal{D}}{\Ex{P\opt}{\rho\left(\bar\pi_S, P\opt \right) \ss \mathcal{D}} - \Ex{P\opt}{\rho(\bar{\pi}_S, P\opt) \ss \mathcal{D} }} = 0 & \text{from \eqref{eq:static_soft_robust} and $\lambda = 0$}\\
    \end{align*}
    The result for $\lambda > 0$ follows using the same steps and the fact that $\cvar[X]\le \E{X}$ for any random variable and therefore:
    \[ \Ex{P\opt}{\rho(\bar{\pi}_S, P\opt) \ss \mathcal{D} } \;\ge\; \bar{\rho}_S   ~.\]
\end{proof}

\section{Proofs: \cref{sec:Soft-Robust-optimization}} \label{app:soft_robust_algorithms}

\begin{proof}[Proof of \cref{prop:static_soft_robust}]
    First, we show that the negations of the terms in the soft-robust objective~\eqref{eq:static_soft_robust} are support functions~\cite{Rockafellar2000optimizationCVAR} of convex sets. For any random variable $X : \Omega \to \real$ with probability measure function $\Pf$, the robust representation of $\cvar$ takes the following form~(e.g.,~\cite{Schied2006Riskmeasures,Rockafellar2000optimizationCVAR}):
    \begin{equation} \label{eq:cr}
        \cvar^{\alpha}[X] = \min_{\xi\in\Delta^\Omega} \left\{ \sum_{\omega\in\Omega} \xi_\omega \cdot X(\omega) \ss \xi_\omega \leq \frac{1}{1-\alpha} \Pf_\omega, \, \forall \omega\in\Omega \right\} ~,
    \end{equation}
    and, therefore, the $\cvar$ term in~\eqref{eq:static_soft_robust} becomes
    \begin{equation} \label{eq:robustform1}
        \cvar^\alpha_{\hat{P} \sim \Pf} \left[\rho(\pi,\hat{P}) \right] = \min_{\xi \in \mathcal{Q}^{\cvar}} \; \sum_{\omega=1}^D \xi_\omega \cdot \rho(\pi,\hat{P}^{\omega})
    \end{equation}
    where the set $\mathcal{Q}^{\cvar}$ is defined as
    \begin{equation*}
        \mathcal{Q}^{\cvar} = \left\{ \xi \in \Delta^D \ss \xi_\omega \leq \frac{1}{1-\alpha} \Pf_\omega, \; \omega \in \Omega \right\} ~.
    \end{equation*}
    As a result of~\eqref{eq:robustform1}, the function $X \mapsto -\cvar_{\hat{P} \sim \Pf}^\alpha [ - X]$ for $X:\Omega\to\Real$ is the support function of set $\mathcal{Q}^{\cvar}$. Note that we are interpreting the random variable $X$ as a vector over $\Real^\Omega$. Similarly, the mean term in \eqref{eq:static_soft_robust} trivially equals to
    \begin{equation} \label{eq:expectation_form}
        \Exp{\hat{P} \sim \Pf}{\rho(\pi,\hat{P})} =
        \min_{\xi \in \mathcal{Q}^{\mathbb{E}}} \; \sum_{\omega\in\Omega} \xi_\omega \cdot \rho(\pi,\hat{P}^{\omega}),
    \end{equation}
    where $\mathcal{Q}^{\mathbb{E}} = \left\{ \Pf \right\}$ is a singleton set. As with the $\cvar$ above, it can be seen readily that the function $X \mapsto -\Exp{\hat{P} \sim \Pf}{ - X }$ for $X:\Omega\to\Real$ is the support function of $\mathcal{Q}^{\mathbb{E}}$.

    Next, any two support functions $f_1(z) = \max_{q\in \mathcal{Q}_1} z\tr q$ and $f_2(z) = \max_{q\in \mathcal{Q}_2} z\tr q$  over convex sets $\mathcal{Q}_1,\mathcal{Q}_2$ satisfy for $\lambda \in [0,1]$~(see for example Chapter 13 of~\cite{rockafellar1970convex}),
    \begin{equation}\label{eq:support_sum}
        \lambda \cdot f_1(z) + (1-\lambda) \cdot f_2(z) = \max_q \; \left\{ q\tr z \ss q\in (\lambda \cdot \mathcal{Q}_1 + (1-\lambda) \cdot \mathcal{Q}_2) \right\}~.
    \end{equation}
    Multiplying the equality in~\eqref{eq:support_sum} by $-1$, and using $-z$ as the parameter, we get:
    \begin{equation} \label{eq:eq_support_sum_minus}
        \begin{aligned}
            - \lambda \cdot f_1(-z) - (1-\lambda) \cdot f_2(-z) &= - \max_q \; \left\{ - q\tr z \ss q\in (\lambda \cdot \mathcal{Q}_1 + (1-\lambda) \cdot \mathcal{Q}_2) \right\} \\
            &= \min_q \; \left\{ q\tr z \ss q\in (\lambda \cdot \mathcal{Q}_1 + (1-\lambda) \cdot \mathcal{Q}_2) \right\}
            ~.
        \end{aligned}
    \end{equation}
    Consider the sets $\mathcal{Q}_1 = \mathcal{Q}^{\cvar}$, $\mathcal{Q}_2 = \mathcal{Q}^{\mathbb{E}}$ and support functions $f_1(X) = -\cvar^\alpha[-X]$ and $f_2(X) = - \E{-X}$ in~\eqref{eq:eq_support_sum_minus}. Then, we can reformulate \eqref{eq:static_soft_robust} as:
    \begin{align*}
        \srob(\pi) &= (1-\lambda) \cdot \E{\rho(\pi,\hat{P})} + \lambda \cdot \cvar^{\alpha}\left[\rho(\pi,\hat{P})\right] \\
        &= \min_{\xi \in \Delta^N} \; \left\{  \sum_{\omega \in\Omega} \xi_\omega \cdot \rho(\pi, \hat{P}^\omega) \ss \xi = \lambda \cdot \xi_1 + (1-\lambda)\cdot \xi_2, \, \xi_1\in \mathcal{Q}^{\cvar},\, \xi_2 \in\mathcal{Q}^{\mathbb{E}} \right\}.
    \end{align*}
    The the feasible set in the equation above in terms in $\xi,\xi_1 \in \Real^N$  (note that $\xi_2 = f$) is represented by these inequalities,
    \begin{align*}
        \xi &= \lambda \cdot \xi_1 + (1-\lambda)\cdot f &
        \xi &\ge \zero &
        \one\tr \xi &= 1 \\
        \xi_1 &\le \frac{1}{1-\alpha} \cdot f &
        \xi_1 &\ge \zero &
        \one\tr \xi_1 &= 1
    \end{align*}
    Now, substituting $\xi_1 = \nicefrac{1}{\lambda} \cdot (\xi - (1-\lambda)\cdot f)$ to the inequalities above, we get
    \begin{align*}
        0&= 0 &
        \xi &\ge \zero &
        \one\tr \xi &= 1 \\
        \nicefrac{1}{\lambda} \cdot (\xi - (1-\lambda)\cdot f) &\le \frac{1}{1-\alpha} \cdot f &
        \nicefrac{1}{\lambda} \cdot (\xi - (1-\lambda)\cdot f) &\ge \zero &
        \one\tr (\nicefrac{1}{\lambda} \cdot (\xi - (1-\lambda)\cdot f)) &= 1 ~,
    \end{align*}
    which, using $\lambda\in[0,1]$ and $f\in\Delta^N$, reduces to:
    \begin{align*}
        0&= 0 &
        \xi &\ge \zero &
        \one\tr \xi &= 1 \\
        \xi &\le \frac{\lambda}{1-\alpha} \cdot f  + (1-\lambda)\cdot f &
        \xi &\ge (1-\lambda)\cdot \Pf &
        0&= 0 ~,
    \end{align*}
    The result then follows by simple algebraic manipulation.
\end{proof}

\begin{proof}[Proof of \cref{prop:milp_static}]
    The proof proceeds by first formulating the soft-robust optimization as a nonconvex quadratic optimization problem and then using the McCormick inequality to turn it to a MILP. Recall that the soft-robust objective in~\eqref{eq:static_soft_robust} is defined as:
    \begin{equation} \label{eq:srobust_1}
        \srob(\pi) \;=\;  (1-\lambda) \cdot \Exp{}{\rho(\pi,\hat{P})} + \lambda \cdot \cvar^{\alpha}\left[\rho(\pi,\hat{P})\right]~.
    \end{equation}
    From the standard definition of $\cvar$, the objective $\srob(\pi)$ becomes
    \begin{equation*} \label{eq:srobust_2}
        \srob(\pi) =  (1-\lambda) \sum_{\omega \in \Omega} \Pf_{\omega} \cdot \rho(\pi, \hat{P}^\omega) + \lambda \cdot \max_{b \in \real} \left(b - \nicefrac{1}{(1-\alpha)}  \sum_{\omega \in \Omega} \Pf_{\omega} \cdot {\max\left\{0, b - \rho(\pi, \hat{P}^\omega)\right\}}\right)\,,
    \end{equation*}
    which can be formulated as the following linear program by introducing a variable $y_\omega = \max \{0, b - \rho(\pi, \hat{P}^\omega) \}$ as:
    \begin{equation} \label{eq:srobust_3}
        \begin{mprog}
            \maximize{y \in \Real^N, b\in \Real} (1-\lambda) \sum_{\omega \in \Omega} \Pf_{\omega} \cdot\rho(\pi, \hat{P}^\omega) + \lambda \cdot \max_{b \in \real} \left(b - \nicefrac{1}{(1-\alpha)}  \sum_{\omega \in \Omega} y_\omega \right) \\
            \stc y_\omega \ge f_\omega \cdot b - f_\omega \cdot \rho(\pi, \hat{P}^\omega), \qquad \forall \omega\in\Omega
            \cs y \ge \zero~.
        \end{mprog}
    \end{equation}

    Next, we express $\rho(\pi, \hat{P}^\omega)$ for each $\omega\in\Omega$ as the following optimization problem based on occupancy frequencies $u$ as follows:
    \begin{equation} \label{eq:lp_milp_unscaled}
        \begin{mprog}
            \rho(\pi,\hat{P}^\omega) = \maximize{u\in\Real^{S\times A}} \sum_{s\in\states} \sum_{a\in\actions} u(s,a) \cdot \sum_{s'\in\states} P(s,a,s') r(s,a,s') \\
            \stc \sum_{a \in \actions} u(s,a) = \sum_{s' \in \states}\sum_{a' \in \actions} \gamma\cdot u(s',a')\cdot P^{\omega}(s',a',s) +  p_0(s), \qquad s \in \states
            \cs u \ge \zero
            \cs u(s,a) = \pi(s,a) \cdot \sum_{a'\in\actions} u(s,a'), \qquad \forall s\in\states, a\in\actions
        \end{mprog}
    \end{equation}
    The linear formulation in~\eqref{eq:lp_milp_unscaled} is based on the dual linear program formulation of an MDP as described in (6.9.2) of~\citet{putterman1994DP}. The last constraint ensures that only occupancy frequencies for $\pi$ are considered and its correctness follows from Theorem 6.9.4 in~\citet{putterman1994DP}. Further, one can scale the constants in~\eqref{eq:lp_milp_unscaled} to get:
    \begin{equation} \label{eq:lp_milp_scaled}
        \begin{mprog}
            f_\omega \cdot \rho(\pi,\hat{P}^\omega) = \maximize{u\in\Real^{S\times A}} \sum_{s\in\states} \sum_{a\in\actions} u(s,a) \cdot \sum_{s'\in\states} P(s,a,s') \cdot r(s,a,s') \\
            \stc \sum_{a \in \actions} u(s,a) = \sum_{s' \in \states}\sum_{a' \in \actions} \gamma\cdot u(s',a')\cdot P^{\omega}(s',a',s) +  f_\omega \cdot p_0(s), \qquad s \in \states
            \cs u \ge \zero~,
            \cs u(s,a) = \pi(s,a) \cdot \sum_{a'\in\actions} u(s,a'), \qquad \forall s\in\states, a\in\actions
        \end{mprog}
    \end{equation}

    Finally, combining~\eqref{eq:srobust_3} with~\eqref{eq:lp_milp_scaled}  we can  formulate the optimization problem $\max_{\pi\in\Pi}$ as follows:
    \begin{equation*} \label{eq:qp_nonconvex}
        \extrarowheight=1mm
        \begin{array}{>{\displaystyle}c>{\displaystyle}r@{\hspace{1.5mm}}>{\displaystyle}l>{\displaystyle}l}
            \operatorname*{maximize}_{\substack{\pi \in  [0,1]^{S \times A}, \, b\in\real,\\
                    u\in\Real^{S\times A\times N}_+,\, y\in \real^N_{+}}}
            &\multicolumn{3}{>{\displaystyle}l}{
                \lambda \cdot \Big(b- \frac{1}{1-\alpha} \sum_{\omega \in \Omega} y(\omega)\Big) + (1-\lambda) \cdot \sum_{s \in \states} \sum_{a \in \actions} \sum_{\omega \in \Omega} u(s,a,\omega) \sum_{s' \in \states} r(s,a,s') \cdot P^{\omega}(s,a,s')}\\
            \operatorname{subject\,to} &y(\omega) - b \cdot f_{\omega} &\geq\; - \sum_{s \in \states} \sum_{a \in \actions}u(s,a,\omega) \sum_{s' \in \states} P^{\omega}(s,a,s') \cdot r(s,a,s'), &\omega \in \Omega, \\
            &\sum_{a \in \actions} u(s,a,\omega) &=\; \sum_{s' \in \states}\sum_{a' \in \actions} \gamma\cdot u(s',a',\omega)\cdot P^{\omega}(s',a',s) + f_{\omega} \cdot p_0(s), &s \in \states,\, \omega \in \Omega, \\
            &\sum_{a\in\actions} \pi(s,a) &=\; 1, &s \in \states, \\
            &u(s,a,\omega) &=\; \pi(s,a) \cdot \sum_{a'\in\actions} u(s,a',\omega),  &s \in \states,\, a \in \actions,\, \omega \in \Omega.
        \end{array}
        \extrarowheight=0mm
    \end{equation*}
    The MILP formulation then follows by upper-bounding the nonlinear constraint
    \[
    u(s,a,\omega) = \pi(s,a) \cdot \sum_{a'\in\actions} u(s,a',\omega)
    \]
    by replacing the right-hand side using the McCormick relaxation~(see, for example, Lemma 4.2 and the argument in~\citet{Petrik2016})
    \[
    u(s,a,\omega) \le  \pi(s,a) \cdot \frac{f_\omega}{1-\gamma}
    \]
    and the fact that $\pi(s,a) \in [0,1]$ and $u(s,a,\omega) \in [0, f_\omega / (1-\gamma)]$~(e.g.~Lemma~C.10 in~\citet{Petrik2010a}). The optimality of the MILP formulation for deterministic policies holds because the McCormick relaxation is precise for the extreme values of the interval when $\pi(s,a) \in\{0,1\}$.
\end{proof}

\begin{proof}[Proof of \cref{thm:s_rect_soft_robust}]
    Without loss of generality, consider a transition model $P \in \mathcal{P}^{D}$ such that $P$ can be represented using a distribution $\zeta \in \Xi$, that is $P = \sum_{\omega\in\Omega} \zeta_{\omega} \hat{P}^{\omega}$. Notice that $\forall s\in\states, \, P_s$ can be written as $\sum_{\omega\in\Omega} \zeta_{\omega} \hat{P}_{s}^{\omega}$. By construction of $\mathcal{P}^R$ and since $\zeta \in \Xi$, it follows that $\forall s \in\states, \, P_{s} \in \mathcal{P}_{s}^R$. Since $\mathcal{P}^{R}$ is a state-wise Cartesian product of $\mathcal{P}_{s}^R$, we can conclude that $\mathcal{P}^D \subseteq \mathcal{P}^R$.

    The second part of the proposition thus immediately follows as
    $\forall \pi \in \Pi, \, \min_{P \in \mathcal{P}^R}  \rho(\pi,P) \leq \min_{P \in \mathcal{P}^D} \rho(\pi,P)$ implies $\rho^R(\pi) \leq \rho^D(\pi) \, \forall \pi\in\Pi$.
\end{proof}

\begin{proof}[Proof of \cref{prop:Bellman_update}]
    Recall the s-rectangular soft-robust MDP objective.
    \begin{align}
        \rrob(\pi) \;=\; \min_{P \in \mathcal{P}^R} \; \rho \left(\pi, P \right)\,.
    \end{align}
    From~\eqref{eq:S-rect}, the S-rectangular soft-robust Bellman optimality operator can be written as
    \begin{align}
        (\Bell_{\mathcal{P}^R} \, v)(s) = \max_{d\in\Delta^A} \min_{P_s \in \mathcal{P}^{R}_s } \; \sum_{a\in\actions} d_a P_{s,a}\tr \left(r_{s,a}  + \gamma \cdot v \right). \label{s-bellman}
    \end{align}
    Using simple algebraic manipulations, we can show that the dual of the right-hand-side of~\eqref{s-bellman} can be formulated as the proposed linear program.
    \begin{equation}
        \begin{aligned}
            &\max_{\substack{d \in \Delta^A,\,b\in\real\\y\in\real_+^{|\Omega|}}} \quad (1- \lambda) \sum_{\substack{a \in \actions\\\omega\in\Omega}} d_a  \Pf_\omega (\hat{P}_{s,a}^{\omega})\tr z_{s,a}  \\
            &\qquad\qquad\qquad + \lambda \Big(b - \frac{1}{1-\alpha} \sum_{\omega\in\Omega} \Pf_\omega \cdot y_\omega \Big) \\
            &\st \qquad y_\omega \geq b- \sum_{a \in \actions} d_a (\hat{P}_{s,a}^{\omega})\tr z_{s,a}, \quad \omega \in \Omega\,. \\
        \end{aligned}
    \end{equation}
    where $z_{s,a} = r_{s,a} + \gamma  \cdot v$.
\end{proof}
\section{Proofs: \cref{sec:approximation_error}} \label{app:bounds_proofs}

In this section, we describe the technical results that underlie the proof of \cref{thm:lower_bound}. The following lemma bounds the difference between a convex combination of occupancy frequencies and the occupancy frequency of the convex combination of transition functions. This serves as the main technical tool when bounding the difference between dynamic and static objectives.
\begin{lem} \label{lem:convex_difference}
    Consider stochastic matrices $P_i \in (\Delta^S)^S, i = 1,\ldots,N$ with occupancy frequencies $h_i = (\eye - \gamma \cdot P_i\tr)^{-1} p_0$. Let $P_\beta = \sum_{i=1}^N \beta_i \cdot P_i$ be the convex combination of $P_i$ for a given $\beta\in\Delta^N$ and let $h_\beta = (\eye - \gamma \cdot P_\beta\tr)^{-1} p_0$ be its occupancy frequency. The convex combination of individual occupancy frequencies is denoted by $e_\beta = \sum_{i=1}^N \beta_i \cdot h_i$. Then:
    \[ \norm{h_\beta - e_\beta }_1 \le \frac{\gamma}{1-\gamma} \cdot \epsilon_1~,\]
    when $\norm{h_i - h_j}_1\le \epsilon_1$ for each $i = 1,\ldots, N$ and $j = 1, \ldots, N$.
\end{lem}
\begin{proof}
    Recall that the following identities hold for the occupancy frequencies~\cite{putterman1994DP}:
    \begin{align} \label{eq:lem_cd_d}
        h_i &= \gamma \cdot P_i\tr h_i + p_0, \quad i = 1, \ldots, N, &
        h_\beta &= \gamma \cdot P_\beta\tr h_\beta + p_0~.
    \end{align}
    Using the identity above and the fact that $\beta\in\Delta^S$, we obtain a similar expression for $e_\beta$:
    \begin{equation} \label{eq:lem_cd_e}
        e_\beta = \gamma \sum_{i=1}^N \beta_i \cdot P_i\tr h_i + p_0~.
    \end{equation}
    Because $h_\beta$ need not be a convex combination of $h_i$ we use the following representation of the difference between $h_\beta$ and the convex combination $e_\beta$ of $h_i$:
    \begin{align*}
        h_\beta - e_\beta &= \gamma\cdot P_\beta\tr h_\beta - \gamma \sum_{i=1}^N \beta_i \cdot P_i\tr h_i &&\text{from \eqref{eq:lem_cd_d} and \eqref{eq:lem_cd_e}} \\
        &= \gamma\cdot P_\beta\tr h_\beta - \gamma\cdot P_\beta\tr e_\beta + \gamma\cdot P_\beta\tr e_\beta - \gamma \sum_{i=1}^N \beta_i \cdot P_i\tr h_i &&\text{add $0$} \\
        &= \gamma\cdot P_\beta\tr h_\beta - \gamma\cdot P_\beta\tr e_\beta + \gamma\cdot \sum_{i=1}^N \beta_i \cdot P_i\tr e_\beta - \gamma \sum_{i=1}^N \beta_i \cdot P_i\tr h_i &&\text{definition of $P_\beta$} \\
        &= \gamma\cdot P_\beta\tr (h_\beta - e_\beta) + \gamma\cdot \sum_{i=1}^N \beta_i \cdot P_i\tr (e_\beta - h_i)~. &&\text{simplify}
    \end{align*}
    Next, subtracting $\gamma\cdot P_\beta\tr (h_\beta - e_\beta)$ from both sides of the equality above, and multiplying by the appropriate matrix inverse leads to:
    \begin{equation}\label{eq:occupancy_difference}
        h_\beta - e_\beta \;=\; \gamma \sum_{i=1}^N \beta_i \cdot (\eye - \gamma\cdot P_\beta\tr)^{-1} P_i\tr (e_\beta - h_i)~.
    \end{equation}
    Applying the $L_1$ norm to both sides of \eqref{eq:occupancy_difference} we get that:
    \begin{align*}
        \norm{h_\beta - e_\beta}_1 &= \norm{\gamma \sum_{i=1}^N \beta_i \cdot (\eye - \gamma\cdot P_\beta\tr)^{-1} P_i\tr (e_\beta - h_i)}_1  \\
        & \leq \gamma \sum_{i=1}^N \beta_i \cdot \norm{(\eye - \gamma\cdot P_\beta\tr)^{-1} P_i\tr (e_\beta - h_i)}_1 &&\text{from triangle inequality} \\
        &\le \gamma \sum_{i=1}^N \beta_i \cdot \norm{(\eye - \gamma\cdot P_\beta\tr)^{-1} P_i\tr}_1 \norm{e_\beta - h_i}_{1} &&\text{from } \norm{Ax}_1 \leq \norm{A}_1\norm{x}_{1} \\
        &\le \gamma \sum_{i=1}^N \beta_i \cdot \norm{(\eye - \gamma\cdot P_\beta\tr)^{-1}}_1 \norm{P_i\tr}_1 \norm{e_\beta - h_i}_{1} &&\text{from $\norm{A B} \le \norm{A} \cdot \norm{B}$} \\
        &\le \gamma \sum_{i=1}^N \beta_i \cdot \norm{(\eye - \gamma\cdot P_\beta\tr)^{-1}}_1 \norm{e_\beta - h_i}_{1} ~.
    \end{align*}
    Then \cref{lem:matrix_norm} combined with the Neumann series representation of matrix inverse implies that $\norm{(\eye - \gamma\cdot P_\beta\tr)^{-1}}_1 \le 1/(1-\gamma)$. It can also be  shown readily by basic algebra that $\norm{e_\beta - h_i}_{1}\le \epsilon_1$. The desired result then follows because $\beta\in\Delta^S$.
\end{proof}

\begin{proof}[Proof of \cref{lem:static_dynamic_difference}]
    Before proving the result, we recall several necessary definitions and identities.
    The static and dynamic returns are defined as
    \begin{subequations} \label{eq:sd_proof_returns}
        \begin{align}
            \label{eq:sd_proof_static}
            \srob(\pi) &= \min_{\xi\in\Xi} \; \Exp{\hat{P} \sim \xi}{\rho \bigl(\pi, \hat{P}\bigr)}  \\
            \label{eq:sd_proof_dynamic}
            \drob(\pi) &= \min_{\xi\in\Xi} \; \rho \left(\pi, \Exp{\hat{P} \sim \xi}{\hat{P}}\right)~.
        \end{align}
    \end{subequations}
    Recall also that the return of a policy $\pi$ in an MDP with the transition matrix $P_\pi \in (\Delta^S)^S$ can be expressed in terms of the occupancy frequency $h_\pi$, defined in~\eqref{eq:occupancy_frequency}, as
    \begin{equation} \label{eq:return_occupancy}
        \rho(\pi,P) = p_0\tr v_\pi = p_0\tr (\eye - \gamma \cdot P_\pi)^{-1} r_\pi =  h_\pi\tr r_\pi ~.
    \end{equation}
    Now, let $\xi^S \in \Delta^{|\Omega|}$ and $\xi^D \in \Delta^{|\Omega|}$ be optimal in~\eqref{eq:sd_proof_static} and~\eqref{eq:sd_proof_dynamic}. Then the soft-robust returns in \eqref{eq:sd_proof_returns} can be expressed in terms of their occupancy frequencies using \eqref{eq:return_occupancy} as
    \begin{equation} \label{eq:return_frequencies}
        \begin{aligned}
            \srob(\pi) &= \Exp{\hat{P} \sim \xi^S}{\rho \bigl(\pi, \hat{P}\bigr)} = \sum_{\omega\in\Omega} \xi^S_\omega \cdot p_0\tr (\eye - \gamma \cdot \hat{P}_\pi^\omega)^{-1} r_\pi  = \sum_{\omega\in\Omega} \xi^S_\omega \cdot (h_\pi^\omega)\tr r_\pi \\
            \drob(\pi) &= \rho \left(\pi, \Exp{\hat{P} \sim \xi^D}{\hat{P}}\right) = p_0\tr \left(\eye - \gamma \sum_{\omega\in\Omega} \xi^D_\omega \cdot \hat{P}_\pi^\omega \right)^{-1} r_\pi = (h_\pi^{\xi^D})\tr r_\pi ~.
        \end{aligned}
    \end{equation}
    where $h_\pi^{\xi^S}$ = $\left(\eye - \gamma \sum_{\omega\in\Omega} \xi^S_\omega \cdot \hat{P}_\pi^\omega \right)^{-1}$ and $h_\pi^\omega = (\eye - \gamma \cdot \hat{P}_\pi^\omega)^{-1}$.

    Next, we get for each $\pi\in\Pi$ that
    \begin{align*}
        \drob(\pi) - \srob(\pi) &\le \Exp{\hat{P} \sim \xi^D}{\rho \bigl(\pi, \hat{P}\bigr)} - \rho \left(\pi, \Exp{\hat{P} \sim \xi^D}{\hat{P}}\right) && \text{From \eqref{eq:return_frequencies} and $\xi^D\in\Xi$} \\
        &= (h_\pi^{\xi^D})\tr r_\pi - \sum_{\omega\in\Omega} \xi^h_\omega \cdot (h_\pi^\omega)\tr r_\pi && \text{From \eqref{eq:return_frequencies}} \\
        &\le \left\lVert h_\pi^{\xi^D} - \sum_{\omega\in\Omega} \xi^D_\omega \cdot h_\pi^\omega \right\rVert_1 \cdot \lVert r_\pi \rVert_\infty && \text{Holder's inequality} \\
        &\le \frac{\gamma \cdot \epsilon_1}{1-\gamma} \cdot r_{\max}  &&\text{From \cref{lem:convex_difference}}~.
    \end{align*}
    Similarly, the reverse inequality follows as
    \begin{align*}
        \srob(\pi) - \drob(\pi) &\le \Exp{\hat{P} \sim \xi^S}{\rho \bigl(\pi, \hat{P}\bigr)} - \rho \left(\pi, \Exp{\hat{P} \sim \xi^S}{\hat{P}}\right) && \text{From \eqref{eq:return_frequencies} and $\xi^S\in\Xi$} \\
        &= \sum_{\omega\in\Omega} \xi^S_\omega \cdot (h_\pi^\omega)\tr r_\pi - (h_\pi^{\xi^S})\tr r_\pi  && \text{From \eqref{eq:return_frequencies}} \\
        &\le \left\lVert \sum_{\omega\in\Omega} \xi^S_\omega \cdot h_\pi^\omega - h_\pi^{\xi^S} \right\rVert_1 \cdot \lVert r_\pi \rVert_\infty && \text{Holder's inequality} \\
        &\le \frac{\gamma \cdot \epsilon_1 (\pi)}{1-\gamma} \cdot r_{\max}  &&\text{From \cref{lem:convex_difference}}~.
    \end{align*}
    Combining the two inequalities above, we obtain that
    \[
    \lvert \srob(\pi) - \drob(\pi) \rvert \;\le\; \frac{\gamma \cdot \epsilon_1 (\pi)}{1-\gamma} \cdot r_{\max}~,
    \]
    which proves the result.
\end{proof}

\begin{proof}[Proof of \cref{thm:error_bound_rectangularization}]
    To establish this bound, define a robust Bellman value operator $\Topt^{\pi,\xi} :  \Real^S \to \Real^S$ for any policy $\pi\in\Pi$, nature's response $\xi\in\Xi$, value function $v\in\Real^S$, and state $s\in\states$ as
    \[
    \left(\Topt^{\pi,\xi} v \right)_s \;=\; \sum_{a\in\actions} \sum_{\omega\in\Omega} \xi_\omega \cdot \pi_{s,a} \cdot (\hat{P}_{s,a}^\omega)\tr (r_{s,a} + \gamma\cdot v) ~.
    \]
    The operator $\Topt^{\pi,\xi}$ is linear and has a unique fixed point $v^{\pi,\xi} \in \Real^S$ which satisfies $\Topt^{\pi,\xi} v^{\pi,\xi} = v^{\pi,\xi}$~\cite{Ho2018FastBellman}.
    Similarly, we define a robust S-rectangular Bellman value operator $\Topt^{\pi} : \real^S \to \real^S$ defined for any policy $\pi\in\Pi$, value function $v\in\Real^S$, and state $s\in\states$ as
    \[
    \left(\Topt^{\pi} v \right)_s \;=\; \min_{\xi\in\Xi} \; \left( \Topt^{\pi,\xi  }v \right)_s~.
    \]
    Note that for a fixed policy $\pi\in\Pi$, the operator $\Topt^\pi$ is equivalent to the Bellman operator in MDPs and satisfies the same properties. Let $\pi^{*}_{D}$ be the optimal policy that optimizes $\rho^D(\pi)$.
    Equipped with the definitions above, we proceed to bound the error $\drob(\pi\opt_D) - \rrob(\pi\opt_D)$. Let $\xi\opt_D$ be the minimizer for $\drob(\pi\opt_D)$ in~\eqref{eq:non_rect_soft_robust} and therefore
    \[
    \drob(\pi\opt_D) \;=\; p_0\tr v^{\pi\opt_D,\xi\opt_D} ~.
    \]
    Similarly, let $\xi\opt_R$ be the minimizer to $\rrob(\pi\opt_D)$ in~\eqref{eq:objective_rectangular} and therefore
    \[
    \rrob(\pi\opt_D) \;=\; p_0\tr v^{\pi\opt_D,\xi\opt_R} ~.
    \]
    Exploiting the fact that $\Topt^\pi$ is an MDP Bellman operator and using standard arguments for MDP value functions (for example, Corollary 4 in \cite{Ho2018FastBellman}) we get that:
    \begin{align*}
        \drob(\pi\opt_D) - \rrob(\pi\opt_D) &= p_0\tr v^{\pi\opt_D,\xi\opt_D} - p_0\tr v^{\pi\opt_D,\xi\opt_R}
        \le \norm{p_0}_1 \cdot \norm{v^{\pi\opt_D,\xi\opt_D} - v^{\pi\opt_D,\xi\opt_R}}_\infty
        = \norm{v^{\pi\opt_D,\xi\opt_R} - v^{\pi\opt_D,\xi\opt_R}}_\infty \\
        &\le \frac{1}{1 - \gamma} \cdot \norm{\Topt^{\pi\opt_D} \,v^{\pi\opt_D,\xi\opt_R} - v^{\pi\opt_D,\xi\opt_R} }_\infty \le \frac{1}{1 - \gamma} \cdot \epsilon_2~,
    \end{align*}
    for the $\epsilon_2$ stated in the theorem. Finally, we employ \cref{lem:difference_any} combined with \cref{thm:s_rect_soft_robust} to show that
    \[
    0 \le \drob(\pi\opt_D) - \rrob(\pi\opt_R) \le \drob(\pi\opt_D) - \rrob(\pi\opt_D) \le \frac{1}{1 - \gamma} \cdot \epsilon_2~,
    \]
    which shows the desired result.
\end{proof}

\begin{proof}[Proof of \cref{thm:lower_bound}]
    The result follows by algebraic manipulation as
    \begin{align*}
        \srob(\pi\opt_S) - \srob(\pi\opt_R) &= \srob(\pi\opt_S) \overbrace{- \drob(\pi\opt_D) + \drob(\pi\opt_D)}^{=0} \overbrace{- \drob(\pi\opt_R) + \drob(\pi\opt_R)}^{=0} - \srob(\pi\opt_R) \\
        &= \underbrace{\srob(\pi\opt_S) - \drob(\pi\opt_D)}_{\text{\cref{lem:difference_any} \& \cref{lem:static_dynamic_difference}}} + \drob(\pi\opt_D) - \drob(\pi\opt_R) + \underbrace{ \drob(\pi\opt_R) - \srob(\pi\opt_R)}_{\text{\cref{lem:static_dynamic_difference}}}\\
        &\le  \frac{2 \gamma \cdot r_{\max} \cdot \epsilon_1 }{1-\gamma}  + \drob(\pi\opt_D) - \drob(\pi\opt_R) \\
        &=  \frac{2 \gamma \cdot r_{\max} \cdot \epsilon_1 }{1-\gamma}  + \underbrace{\drob(\pi\opt_D) - \drob(\pi\opt_R)}_{\text{\cref{thm:error_bound_rectangularization}}} \\
        &\le  \frac{2 \gamma \cdot r_{\max} \cdot \epsilon_1 }{1-\gamma}  + \frac{\epsilon_2}{1-\gamma} ~.
    \end{align*}
\end{proof}

\section{Experimental Details}
\label{app:experimental_results}
\subsection{Baselines}
We describe below the two custom baselines algorithms Soft-Robust Soft Actor-Critic (SR-SAC) and Robust Soft Actor-Critic (R-SAC) algorithms that we use for comparing the performance of the dynamic soft-robust objective.
The SR-SAC algorithm extends the Soft Actor-Critic (SAC)~\cite{haarnoja18bSA} to use soft-robust updates~\cite{Derman2018Soft-RobustAC}. Similarly, the R-SAC algorithm extends the SAC algorithm to use robust updates~\cite{Mankowitz2020RobustRL,iyengar2005RMDP}.
The SAC algorithm is a variant of the standard policy iteration algorithm that learns maximum-entropy optimal policies. We note that the soft-robust and robust updates only affect the policy-evaluation step of the SAC algorithm. Hence, we will only describe the change in the policy evaluation step.
In the policy evaluation step, the SAC algorithm estimates the value function and action-value function of a policy according to an objective that maximizes the future expected returns and entropy of the optimal policy.

Let $V : \real^{S} \to \real^{S}$ and $Q : \real^{S\times A} \to \real^{S \times A}$ denote the value function and action-value function of policies respectively.
We will refer to the function approximators used by the SAC algorithm to represent the action-value function and the value function as the Q-network and V-network respectively.

The SAC algorithm optimizes the Q-network and V-network in an off-policy manner to minimize the soft-Bellman residual error.
\begin{align}
    J(Q) = \mathbb{E}_{s_t ,a_t \sim D}[(Q(s_t,a_t) - \bar{Q}(s_t,a_t))^2] \\
    \text{where } \bar{Q}(s_t,a_t) =   \mathbb{E}_{s_{t+1} \sim D}[r_{s_t,a_t,s_{t+1}} + \gamma \bar{V}(s_{t+1})] \\
    J(V) = \mathbb{E}_{s_t \sim D}[(V(s_t) - \mathbb{E}_{a_t \sim \pi(s_t)}[Q(s_t,a_t)-  \log(\pi(s_t,a_t))]^{2}]
\end{align}
where $D$ is data collected using some behavior policy $\pi\in\Pi$ and $\bar{V}(s_t)$ is the value function estimated at state $s_t$ as predicted by a target value function network, used for stabilizing training~\cite{haarnoja18bSA}. The weights of the target value function are updated as exponentially moving weighted average of the weights of the V-network.


In the SR-SAC algorithm, we optimize the Q-network to minimize the soft-robust Bellman residual error.
\begin{align}
    J_{soft-robust}(Q) = \mathbb{E}_{s_t ,a_t \sim \bar{P}}[(Q(s_t,a_t) - \bar{Q}(s_t,a_t))^2] \\
    \text{where } \bar{Q}(s_t,a_t) = \sum_{\omega \in \Omega} \Pf_{\omega} \cdot \mathbb{E}_{s_{t+1} \sim \hat{P}^{\omega}(s_t,a_t)}[r_{s_t,a_t,s_{t+1}} + \gamma \bar{V}(s_{t+1})] \\
    J_{soft-robust}(V) = \mathbb{E}_{s_t \sim \bar{P}}[(V(s_t) - \mathbb{E}_{a_t \sim \pi(s_t)}[Q(s_t,a_t)-  \log(\pi(s_t,a_t))]^{2}]
\end{align}
Notice that, in this case, the data samples for the updates are collected by simulating the nominal model $\bar{P}$.

Similarly, in the R-SAC algorithm, we optimize the Q-network to minimize the robust Bellman residual error.
\begin{align}
    J_{robust}(Q) = \mathbb{E}_{s_t ,a_t \sim \bar{P}}[(Q(s_t,a_t) - \bar{Q}(s_t,a_t))^2] \\
    \text{where } \bar{Q}(s_t,a_t) = \min_{\omega\in\Omega} \mathbb{E}_{s_{t+1} \sim \hat{P}^{\omega}(s_t,a_t)}[r_{s_t,a_t,s_{t+1}} + \gamma\bar{V}(s_{t+1})] \\
    J_{robust}(V) = \mathbb{E}_{s_t \sim \bar{P}}[(V(s_t) - \mathbb{E}_{a_t \sim \pi(s_t)}[Q(s_t,a_t)-  \log(\pi(s_t,a_t))]^{2}]
\end{align}

\subsection{Population Domain}

This MDP consists of $51$ states, each represents the current pest population as determined by trapping ($0$ means no pest population). There are $5$ actions available, with each action representing the use of an increasingly potent pesticide. The true transition probabilities are based on a logistic model of population growth as described in~\cite{Tirizoni2018PolicyC}. The discount factor is $\gamma = 0.9$.

To compute the posterior distribution over $\hat{P}$, we gather $300$ state-action transition samples from a single episode. Using these transition samples, we fit an exponential population model~\cite{Kery2012Bayesian} using the JAGS modeling language~\cite{Plummer2003JAGS} and sample $100$ posterior samples using MCMC. We use these samples to formulate and solve the MILP in \cref{fig:MILP} and to run \cref{alg:norbu_s}. We use confidence $\alpha = 0.7$ for both the percentile criterion and soft-robust objective for the evaluation. We also use $\lambda = 0.5$ for the soft-robust objective. We use 100 samples from the posterior distribution both to compute and evaluate the methods' returns.

\subsection{Cancer Simulator}
The cancer simulator models the growth of tumors in cancer patients. The state is a 4-dimensional vector that captures the dynamics of the tumor's growth. The monthly binary action determines whether to administer chemotherapy to the patient~\cite{gottesman2020interpretable,Ribba2012ATG}. The discount factor $\gamma$ is set to 0.9.

\begin{table}
    \centering
    \begin{tabular}{ p{6cm}|r }
        \toprule
        \multicolumn{1}{c|}{\emph{Parameter}} & \multicolumn{1}{c}{\emph{Value}} \\
        \midrule
        Features & 2nd order polynomial \\
        Policy learning rate & 3e-4 \\
        Q-value network learning rate & 3e-4 \\
        V-network learning rate & 3e-4 \\
        Train iterations & 2000 \\
        Episodes per iteration & 30 \\
        Test episodes per transition model & 100 \\
        Train transition models & 50 \\
        Test transition models & 100 \\
        States sampled per model and update & 100 \\
        Batch size & 150 \\
        Target update rate & 0.01 \\
        Discount factor & 0.9 \\
        Hidden layers & (400,300) \\
        Activation & Relu \\
        \bottomrule
    \end{tabular}
    \caption{Cancer Simulator: SR-SAC and R-SAC} \label{tab:params_sac}
\end{table}

We model the true transition probability model $P_{s_t,a_t}$ as a multivariate Normal random variable with mean $w \in \real^l$ and diagonal covariance matrix $\Sigma \in \real^{l \times l}$. The mean and variance are linearly weighted functions of state features. We sample a batch of data consisting of $600$ samples (20 trajectories) using the cancer simulator with transition noise=0.03 and the $\epsilon$-greedy behavior policy provided by~\cite{gottesman2020interpretable} with $\epsilon=0.1$. Using the sampled data, we train a multivariate Bayesian linear regression model to predict the posterior distribution of weights $w$ and the covariance matrix $\Sigma$. We assume a Normal prior $\mathcal{N}(0,1)$ for each element of the weight vector $w$ and a $\operatorname{HalfNormal}(0.001)$ prior for the elements of the covariance matrix $\Sigma$. We construct the train uncertainty set as shown in \cref{alg:norbu_s} by sampling 50 weight vectors and covariance matrices from the posterior distribution using the MCMC algorithm~\cite{hoffman2011nouturn}.
We similarly construct the test uncertainty set by sampling 100 weight vectors and covariance matrices from the posterior distribution. We keep the test and the train sets consistent across all the experiments on the cancer simulator. \Cref{tab:params_sac,tab:params_srvi} summarizes the parameters of the methods we compare.

\begin{table}
    \centering
    \begin{tabular}{p{6cm}|r }
        \toprule
        \multicolumn{1}{c|}{\emph{Parameter}} & \multicolumn{1}{c}{\emph{Value}} \\
        \midrule
        Features & 2nd order polynomial \\
        Train iterations & 150 \\
        Episodes per iteration & 30 \\
        Test episodes per transition model & 100 \\
        Train transition models & 50 \\
        Test transition models & 100 \\
        Batch size & 150 \\
        States sampled per model and update & 100 \\
        Discount factor & 0.9 \\
        \bottomrule
    \end{tabular}
    \caption{Cancer Simulator: SRVI} \label{tab:params_srvi}
\end{table}
\subsection{Additional Experiments}
\subsubsection{Inventory Management Problem}

This domain models a common dilemma encountered by retailers while procuring goods for future sales. We assume an infinite horizon time period during which a retailer procures and sells only 1 type of item. The procured goods are stored in an inventory of limited capacity. Goods stored in the inventory have a fixed holding cost-per-unit per time step. The demands received by the retailer at time $t$ are served at time $t+1$ with goods available in the inventory. Any demand that is not satisfied is backlogged with a fixed backlog cost-per-unit per time step. At every time step, the retailer attempts clearance of as many backlogged demands as possible with the available inventory. The states in this context represent the current quantity of goods in the inventory and actions represent the orders placed by the retailer.
In our setup, the maximum order quantity is $40$ units and the inventory has a fixed capacity of 50 units. The minimum demand is $0$ and the maximum demand is $50$. For the sake of simplicity, we disable backlogging in our experiments. The variable cost, per-unit purchase price, holding cost, backlog cost, and sales price are $2.49$,$3.99$, $0.1$, $0.15$ and $4.99$ respectively. We set risk level $\alpha=0.8$ and discount factor $\gamma=0.99$.
The reward at any time step is the profit incurred from the sales. The demand for goods per time step is stochastic, which in turn makes the transitions stochastic. We assume that the demand comes from the Poisson distribution with an unknown rate $\lambda$. We assume that the true value of $\lambda$ is $10$.

To generate a batch of data, we use a sampling policy that always purchases the maximum available goods. This enables us to sample uncensored demands which in turn allows us to compute the posterior distribution of the demands analytically.
We model the posterior distribution of demands as a Gamma distribution and assume a Gamma prior with shape=4 and scale=6. We fit the posterior distribution using a batch of 50 transitions obtained using the sampling policy and true demand distribution. We sample 100 demand models from the posterior distribution, for training and testing the SRVI RL agent against other baseline RL agents. We assume that the initial distribution is uniform across all the states.

~\cref{fig:inventory} compares the performance of the dynamic soft-robust algorithm namely the SRVI algorithm with RSVF~\cite{russel2019BeyondCR} (percentile-criterion) and Bayesian Confidence Region~\cite{russel2019BeyondCR} for different values of $\lambda$. Since BCR is not meant to optimize the mean returns, we set its $\lambda$ value to 0 throughout the experiment.
Notice that, as expected, the mean performance of the SRVI agent increases with a decrease in $\lambda$ and the $\cvar$ performance increases with an increase in $\lambda$. Further, BCR performs very poorly as compared to the SRVI and RSVF agents. This behavior is not surprising since in contrast to RSVF and SRVI, BCR ambiguity sets are constructed from confidence regions which are often very large and make policies unnecessarily conservative. Although the mean performance of the SRVI and RSVF agent are comparable in this domain, the SRVI agent outperforms both, BCR and RSVF in robust performance at $\lambda=1.0$.
\begin{figure}[h!]
    \centering
    \includegraphics[width=0.5\linewidth]{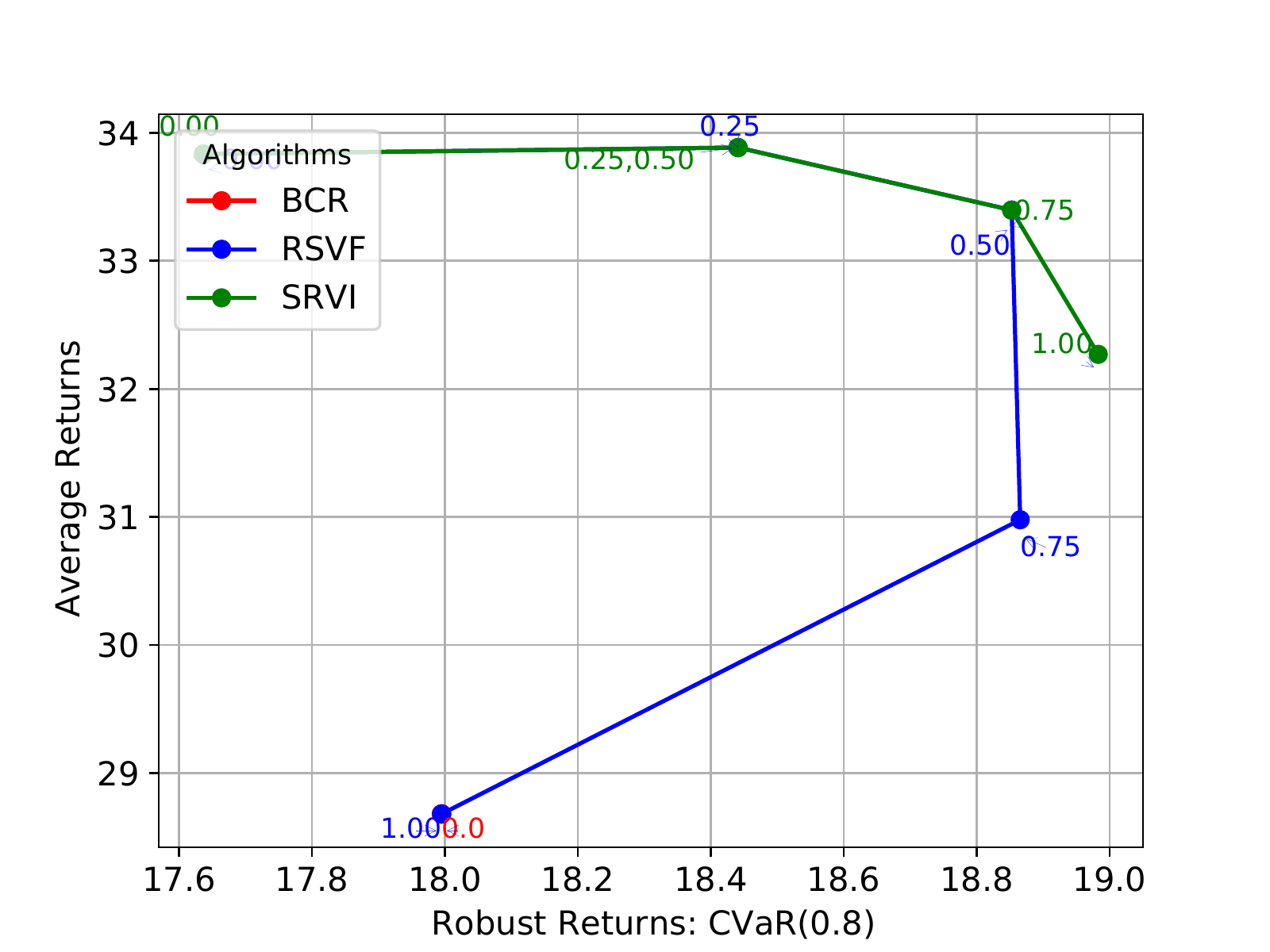}
    \caption{Comparison of Mean and Robust Performance of SRVI, RSVF and BCR in Inventory domain for different values of $\lambda$ indicated by the overlay text. RSVF outperforms RSVF and BCR for appropriately chosen $\lambda$ values.} \label{fig:inventory}
\end{figure}

\subsubsection{Riverswim}
The Riverswim domain is inspired from the Riverswim domain in~\cite{Strehl2004Exploration}. This domain is modeled as 20-states 2-actions MDP with discount factor $\gamma=0.95$. The states $s_1, \ldots, s_{20}$ represent the current position of the agent in the river and the actions $a_1$ and $a_2$ represent the act of swimming 1 unit in the direction of the river's current and 1 unit against the direction of the river's current respectively. The direction of the river's current is from $s_{20} \xrightarrow[]{} s_1$. Choosing action $a_1$ in $s_{i}$ results in transitioning to the state $s_{i-1}$ with probability 1. On the other hand, if the agent chooses $a_2$, it will transition to the state $s_{i+1}$ with probability $0.2$, or to the state $s_{i-1}$ with probability $0.5$ or stay in $s_{i}$ itself with probability $0.3$. In states where $s_{i-1}$ or $s_{i+1}$ is undefined, the agent will continue to stay in the current state with the respective probability. The reward received on reaching state $s_{20}$ is +100. The agent also received a reward of +5 each time it moves 1 step closer to state $s_{20}$. Hence to maximize its returns, the agent has to swim towards state $s_{20}$ i.e., against the direction of the river's current. We assume that the initial distribution is uniform across states.

The posterior distribution over $\bar{P}$ is modeled as a Dirichlet Distribution while assuming a uniform Dirichlet prior. We sample 15 state-action transition samples from a single episode and use them for analytically computing the concentration parameters of the posterior Dirichlet distribution. We sample 100 transition models each from the posterior distribution for training and testing purposes.

~\cref{fig:riverswim} compares the performance of the dynamic soft-robust algorithm namely the SRVI algorithm with RSVF~\cite{russel2019BeyondCR} (percentile-criterion) and Bayesian Confidence Region~\cite{russel2019BeyondCR}.
Qualitatively, the results on this domain mirror what we have observed in the Inventory domain. Again, SRVI outperforms BCR and is comparable to RSVF in mean performance. However, SRVI achieves the best robust performance at $\lambda=0.75$.
\begin{figure}[h!]
    \centering
    \includegraphics[width=0.5\linewidth]{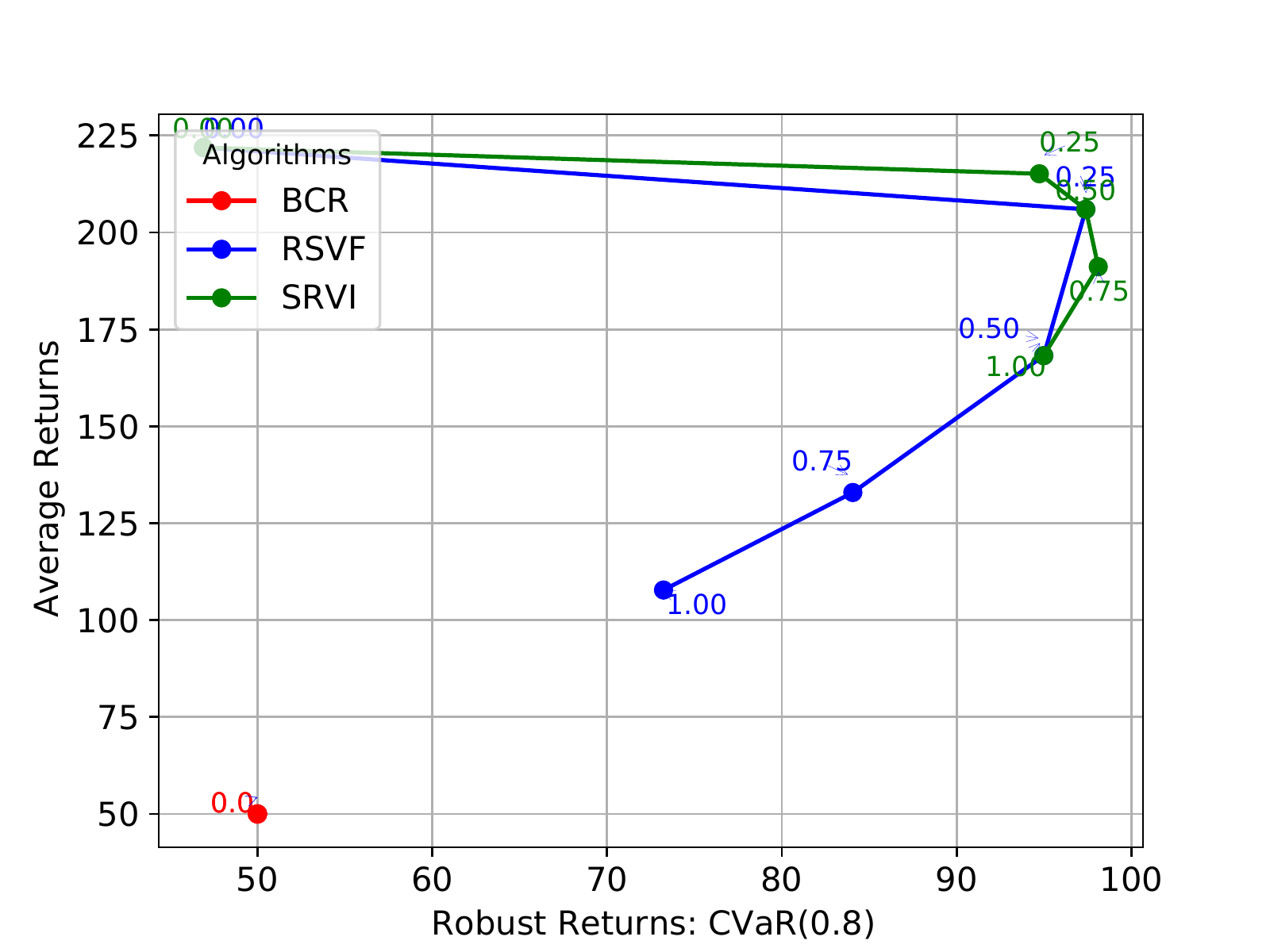}
    \caption{Comparison of Mean and Robust Performance of SRVI, RSVF and BCR in Riverswim domain for different values of $\lambda$ indicated by the overlay text. RSVF outperforms RSVF and BCR for appropriately chosen $\lambda$ values.} \label{fig:riverswim}
\end{figure}

\subsection{Code Details}
Since our code is heavily dependent on one of our in-house libraries which cannot be easily de-anonymized, we will release the code after the paper is published.

\section{Related Work}
\label{sec:Related work}

Numerous robust objectives for mitigating model uncertainty have been proposed in the literature. We discuss a number of them in more detail in this section.

\paragraph{Dynamic robust objectives.} A vast majority of work in Robust RL has studied objectives that assume a dynamic uncertainty model for achieving tractability. \citeasnoun{Mankowitz2020RobustRL} proposed robust algorithms that optimize entropy-regularized policies against the worst model in the uncertainty set. While these algorithms scale to continuous state and action spaces, they do not provide any kind of probabilistic guarantees on the expected returns like our framework and compute overly conservative policies. On the other hand,~\citep{Derman2018Soft-RobustAC} proposed soft-robust actor-critic that optimizes only the mean of the expected returns computed for a fixed distribution over models in the uncertainty set.

In contrast to the prior work, our dynamic soft-robust algorithm dynamically computes the distribution over uncertain models that provide guarantees on the user-specified quantile of the expected returns for the optimal policy. \citeasnoun{Derman2019BayesianRobust} introduced scalable algorithms that optimize an RMDP objective while accounting for changing dynamics. This framework also suffers from the shortcomings of the percentile criterion. Another related work~\cite{Xu2012Distributionally} constructs a plausible framework to incorporate any probabilistic information about the uncertain models in RMDPs and shows a connection between coherent risk measures and distributionally-robust MDPs. However, their main objective is different from ours as they do not aim to address the shortcomings of the percentile criterion. Finally, in the same vein as our work, policy gradient methods for optimizing $\cvar$ of expected returns have been studied by~\citeasnoun{Hiraoka2019LearningRO}. Nonetheless, these methods~\cite{Hiraoka2019LearningRO} do not exploit the coherent properties of this measure and only tend to find local optimal policies.

\paragraph{Static robust objectives.} Few works have focused on optimizing robust objectives while retaining the static uncertainty model assumption~\cite{bucholz2020CMFMP,Steimle2018MM,bucholz2019CMDP,merakli2019RiskA}. However, we note that the robust objectives used in these works are quite different than ours. \citeasnoun{Steimle2018MM} proposed a mixed-integer linear program and a fast heuristic algorithm to optimize the weighted expected returns across different models in a finite-horizon setting, whereas our objective optimizes the policy for the worst distribution over models in the ambiguity set.~\citeasnoun{bucholz2020CMFMP} uses the same objective as in~\citeasnoun{Steimle2018MM}, but considers both finite and infinite-horizon settings. The authors of~\citeasnoun{Steimle2018MM} proposed a MILP for calculating the exact deterministic policy in the finite-horizon setting, and other approximation algorithms that optimize a finite class of randomized Markovian policies for the infinite-horizon case. In another similar work, \citeasnoun{merakli2019RiskA} proposed an approximate MILP for optimizing the percentile-criterion. However, since the original objective is non-convex, the approximation may not generate optimal deterministic solutions.

\paragraph{Ambiguity set optimization.} Some related work has considered partial correlations between uncertain model parameters to mitigate the conservativeness of learned policies~\cite{Derman2019BayesianRobust,Mannor2016KRect,goyal2020BeyondR,Mannor2012LightningDN}. Examples of such works are k-rectangular~\cite{Mannor2016KRect,Mannor2012LightningDN} and r-rectangular~\cite{goyal2020BeyondR} ambiguity sets. These approaches mitigate the conservativeness of S- and SA-rectangular ambiguity sets by capturing correlations between the uncertainty and by limiting the number of times the uncertain parameters deviate from the mean parameters. Despite this progress, most of this works still relies on weak statistical concentration bounds for the construction of ambiguity sets, which can make the ambiguity sets unnecessarily large and result in conservative policies. In contrast, the soft-robust ambiguity sets are convex and can be precisely constructed without using concentration bounds. Therefore, the soft-robust ambiguity sets are relatively tighter and result in learning less-conservative solutions.
\end{document}